\documentclass[jcs]{iosart2x}

\usepackage{graphicx}
\usepackage{amsmath,amssymb,amsfonts}
\usepackage{textcomp}
\usepackage{enumerate}
\usepackage{url}
\usepackage{times}
\usepackage{latexsym}
\usepackage{enumitem}
\usepackage{algorithm,algorithmic}
\usepackage{multirow}
\usepackage{booktabs}
\usepackage{array}
\usepackage{stfloats}
\usepackage{subfigure}
\usepackage{color}
\usepackage{dcolumn}
\usepackage{endnotes}
\usepackage{multicol}
\usepackage{ulem}
\usepackage{lineno}

\newtheorem{theorem}{Theorem}

\pubyear{0000}
\volume{0}
\firstpage{1}
\lastpage{1}

\begin{document}
\makeatletter
\let\put@numberlines@box\relax
\makeatother
\begin{frontmatter}

\title{Learning transferable and discriminative features for unsupervised domain adaptation}
\runningtitle{Learning transferable and discriminative features for unsupervised domain adaptation}



\author[A]{\inits{N.}\fnms{Yuntao Du} \ead[label=e1]{duyuntao@smail.nju.edu.cn}}%
and
\author[A]{\inits{N.N.}\fnms{Ruiting Zhang} \ead[label=e2]{rayting1111@gmail.com}}
and
\author[A]{\inits{N.-N.}\fnms{Xiaowen Zhang} \ead[label=e3]{zhangxw@smail.nju.edu.cn}}
and
\author[A]{\inits{N.N.}\fnms{Yirong Yao} \ead[label=e4]{yaotx@smail.nju.edu.cn}}
and
\author[A,B]{\inits{N.-N.}\fnms{Hengyang Lu} \ead[label=e5]{ luhengyang@jiangnan.edu.cn}}
and
\author[A]{\inits{N.-N.}\fnms{Chongjun Wang} \ead[label=e6]{chjwang@nju.edu.cn} \thanks{Corresponding author. \printead{e6}.}}
\runningauthor{Yuntao Du et al.}
\address[A]{
	State Key Laboratory for Novel Software Technology, \orgname{Nanjing University},
Nanjing 210023, \cny{China}\printead[presep={\\}]{e1,e2,e3,e4,e6}}
\address[B]{School of Artificial Intelligence and Computer Science, \orgname{Jiangnan University}, Wuxi, \cny{China}\printead[presep={\\}]{e5}}

\begin{abstract}
	Although achieving remarkable progress, it is very difficult to induce a supervised classifier without any labeled data. Unsupervised domain adaptation is able to overcome this challenge by transferring knowledge from a labeled source domain to an unlabeled target domain. Transferability and discriminability are two key criteria for characterizing the superiority of feature representations to enable successful domain adaptation. In this paper, a novel method called \textit{learning TransFerable and Discriminative Features for unsupervised domain adaptation} (TFDF) is proposed to optimize these two objectives simultaneously. On the one hand, distribution alignment is performed to reduce domain discrepancy and learn more transferable representations. Instead of adopting \textit{Maximum Mean Discrepancy} (MMD) which only captures the first-order statistical information  to measure distribution discrepancy, we adopt a recently proposed statistic called \textit{Maximum Mean and Covariance Discrepancy} (MMCD), which can not only capture the first-order statistical information  but also capture the second-order statistical information  in the reproducing kernel Hilbert space (RKHS). On the other hand, we propose to explore both local discriminative information via manifold regularization and global discriminative information via minimizing the proposed \textit{class confusion} objective to learn more discriminative features, respectively. We integrate these two objectives into the \textit{Structural Risk Minimization} (RSM) framework and learn a domain-invariant classifier. Comprehensive experiments are conducted on five real-world datasets and the results verify the effectiveness of the proposed method.
\end{abstract}


\end{frontmatter}


\section{Introduction}

Supervised learning has achieved remarkable progress in many fields with the help of a large number of labeled training samples \cite{Wu2007Top1A}. However, when there are few and even no labeled samples, it is difficult to, if not impossible, induce a supervised classifier. Rather, there is a need for versatile algorithms that reduce the need for large labeled  datasets across multiple domains. Unsupervised domain adaptation address this need by transferring knowledge from a different but related domain (source domain) with labeled samples to a target domain with unlabeled samples to improve the performance of the target domain \cite{ref_5}.  For example, an object classification model trained on manually annotated images may not generalize well to new images obtained under substantial variations in pose, occlusion, or light. Domain adaptation aims to enable knowledge transfer from the labeled source domain to the unlabeled target domain by exploring domain-invariant features that bridge different domains \cite{ref_8}.

\begin{figure*}[tbp]
	\centering
	\subfigure[Source by source-only model]{
	\includegraphics[width = 0.43\textwidth]{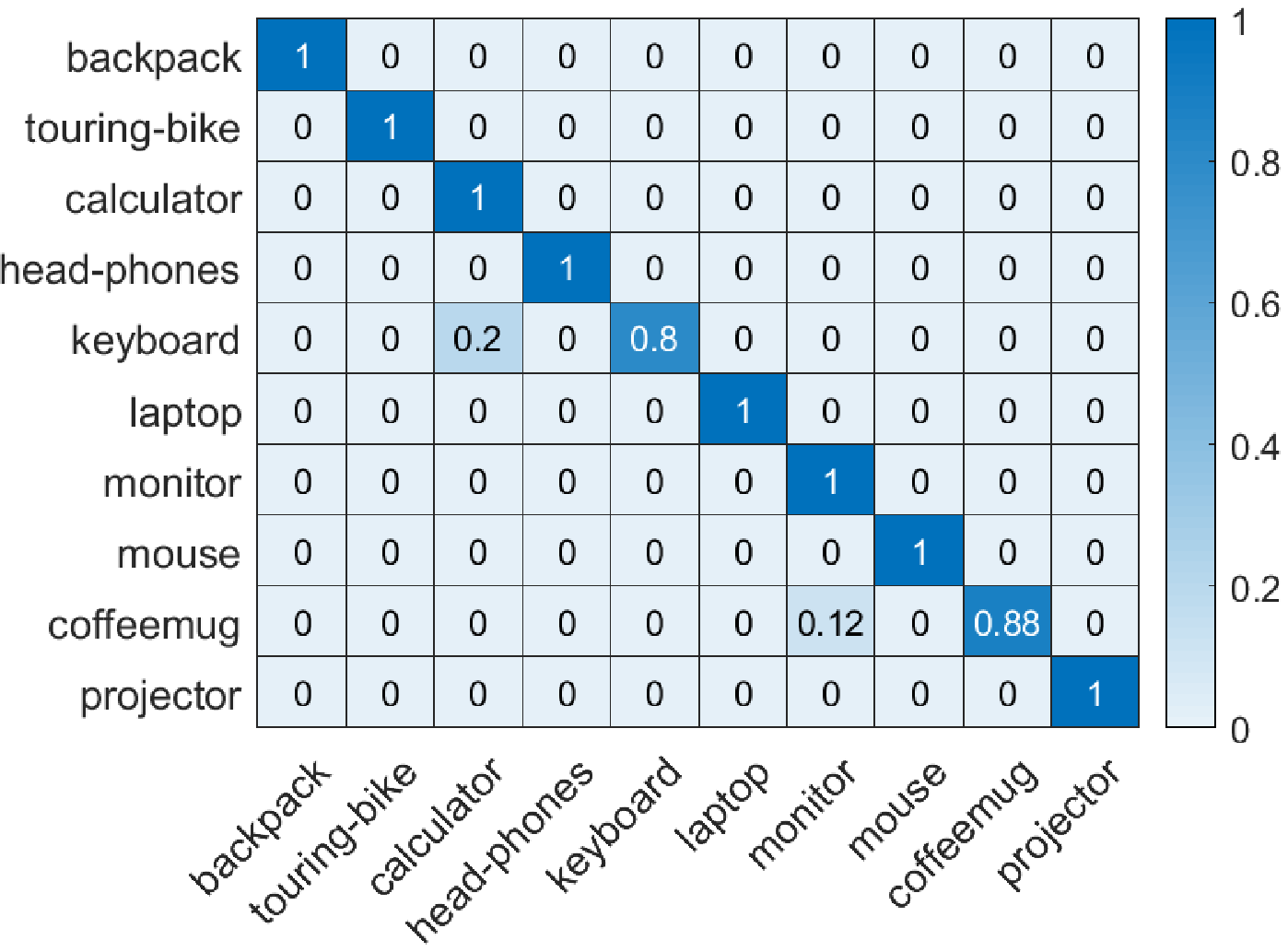}
	}
	\subfigure[Target by source-only model]{
	\label{fig:subfig:a}
	\includegraphics[width = 0.43\textwidth]{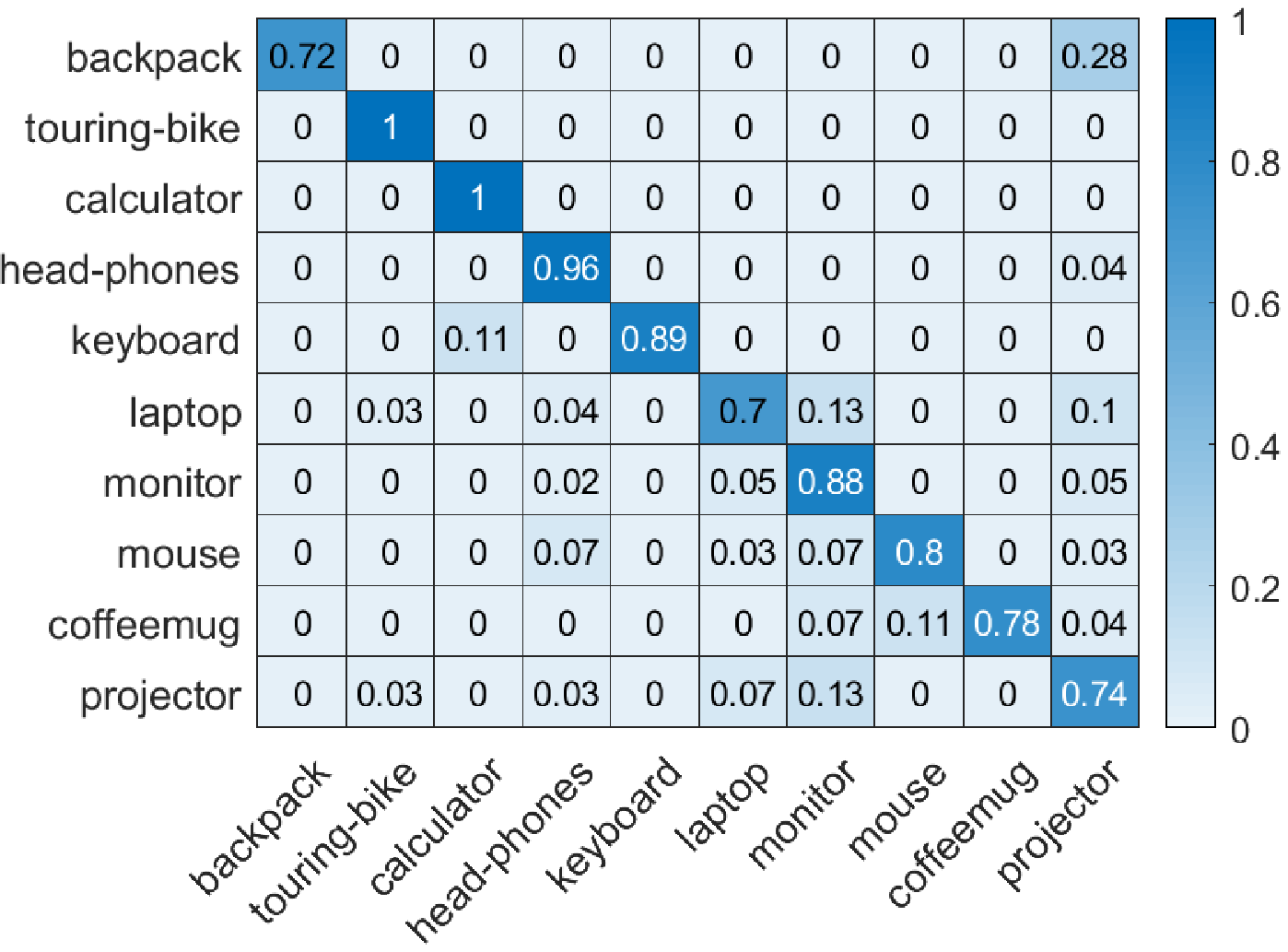}
	}
	\subfigure[target by MEDA \cite{ref_10}]{
	\includegraphics[width = 0.43\textwidth]{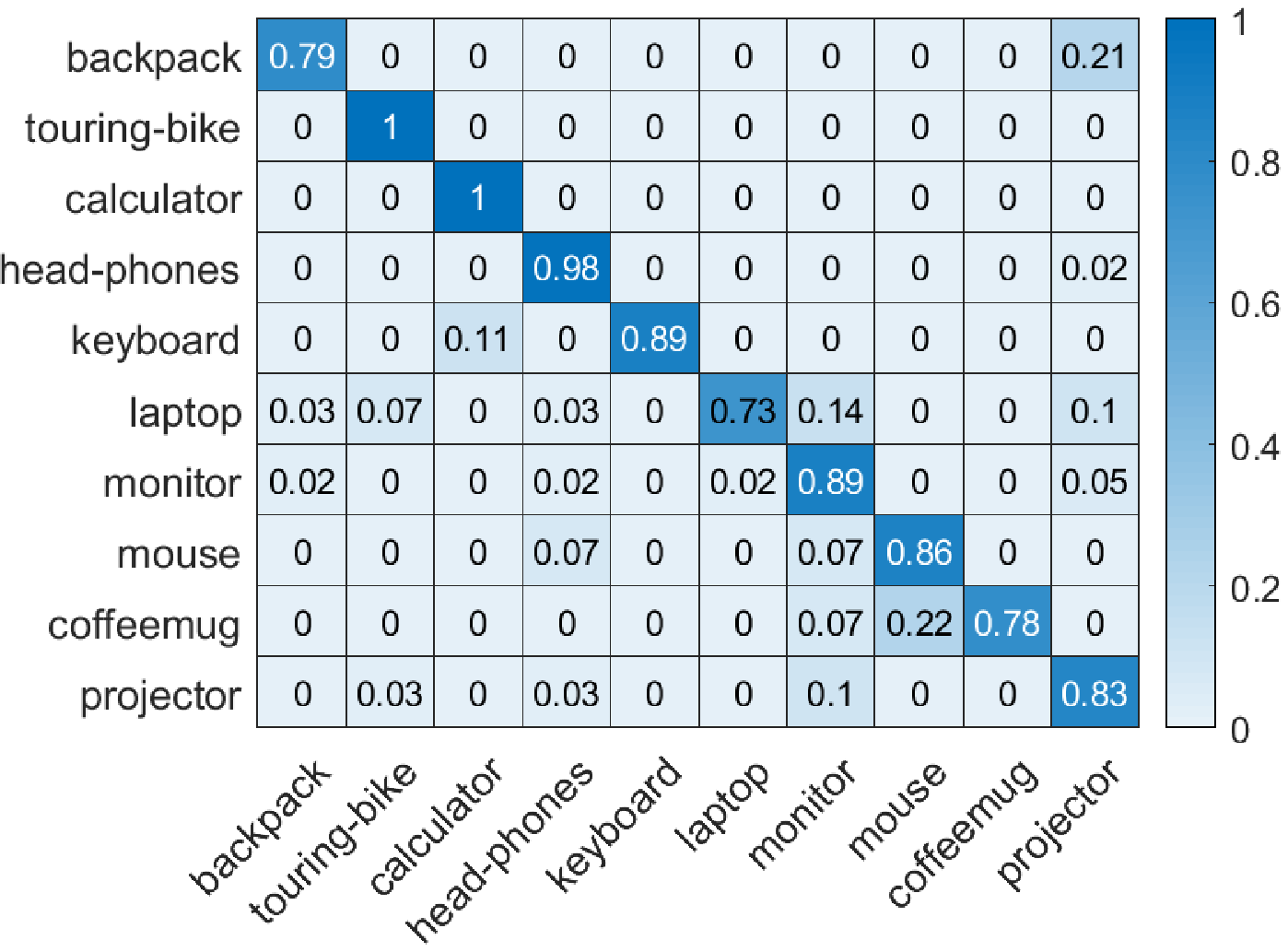}
	}
	\subfigure[target by TFDF]{
	\label{fig:subfig:a}
	\includegraphics[width = 0.43\textwidth]{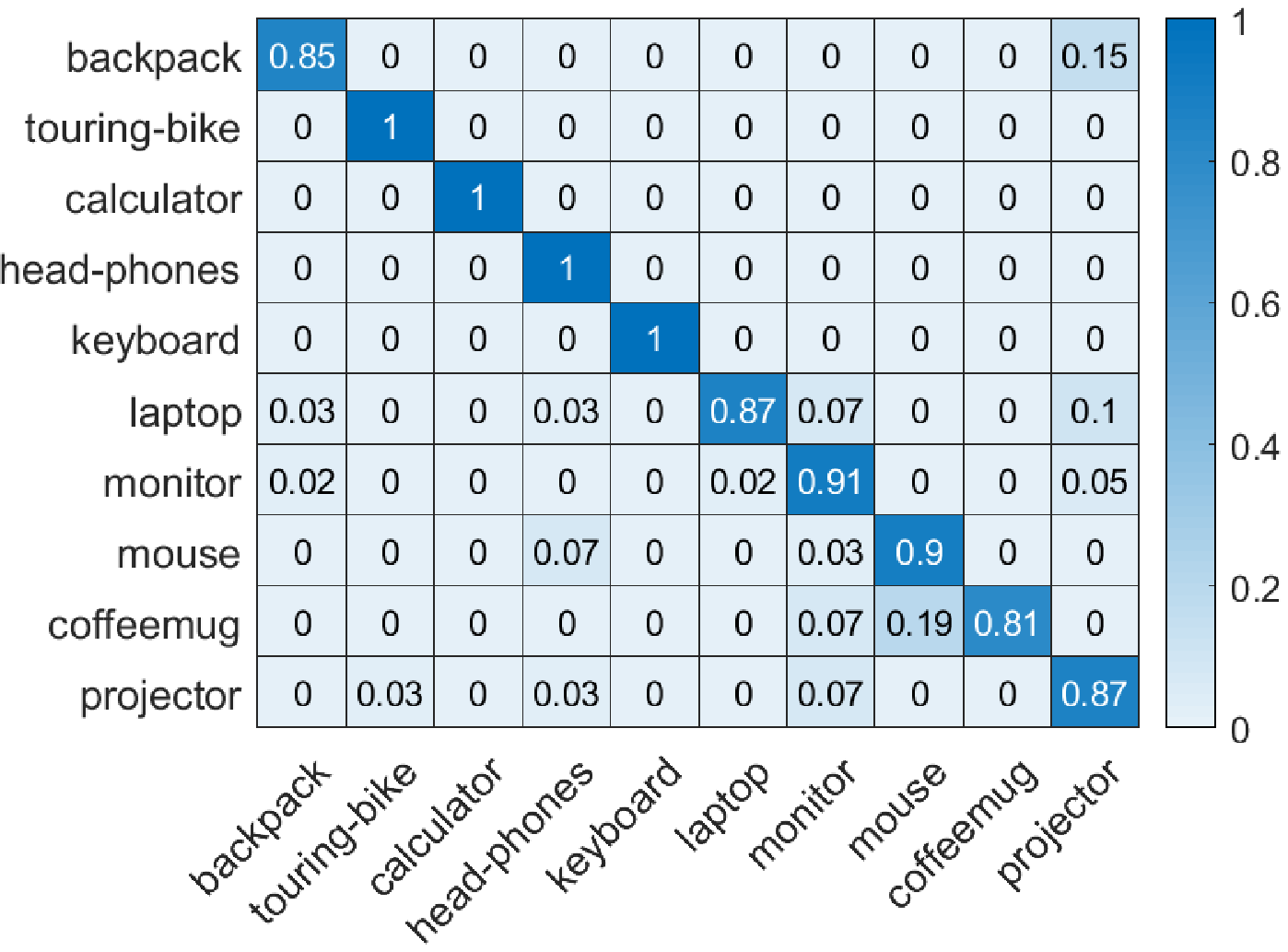}
	}
	\caption{The error matrix of different methods on tasks D $\rightarrow$ W of Offcie-Caltech dataset. Source-only model means the model trained with only labled data in the source domain. (a) Source-only model tested in the source domain (b) Source-only model tested in the target domain, (c) Model trained with MEDA (d) Model trained with TFDF. The results show that TFDF can effevtively avoid class confusion in the target domain and performs much better than other methods.}
	\label{exam}
\end{figure*}

Transferability and discriminability are two key criteria that characterize the superiority of feature representations to enable domain adaptation \cite{ref_30,ref_24,ref_8,ref_9,ref_10}. The transferability indicates the ability of feature representations to bridge the discrepancy across domains, and we can effectively transfer a learning model from the source domain to the target domain via the transferable feature representations \cite{ref_8,ref_9,ref_10}. Discriminability refers to the ability to separate different categories easily by a supervised classifier trained on the feature representations, and the model can achieve better classification performance via the discriminative feature representations \cite{ref_24,ref_30}.

Since the source samples and target samples are drawn from different distributions, it is important to reduce the distribution discrepancy across domains to learn transferable features. The mostly used shallow domain adaptation approaches include instance reweighting \cite{ref_6,ref_7} and distribution alignment \cite{ref_8,ref_9,ref_10}. The former assumes that a certain portion of the samples in the source domain can be reused for learning in the target domain and the samples from the source domain can be reweighted according to the relevance to the target domain. While the latter assumes that there exists a common space where the distributions of two domains are similar and focus on finding a feature transformation that projects features of two domains into another common subspace with less distribution discrepancy \cite{ref_8,ref_9,ref_10}. \textit{Maximum Mean Discrepancy} (MMD) \cite{ref_11} based methods are popular methods for distribution alignment, where the MMD distance is used to evaluate the distribution discrepancy across domains.

While achieving remarkable progress, the experiments in \cite{ref_30} indicate that previous domain adaptation methods tend to enhance the transferability at the expense of deteriorating the discriminability. Thus, some methods, including geometrical based methods \cite{ref_26} and manifold regularization based methods \cite{b33,ref_24}, also aim to improve the discriminability of the feature representations. Geometrical based methods consider the geometric divergence between both domains or the variance information in the target domain. Manifold regularization based methods are inspired by \textit{manifold assumption} \cite{ref_19}, which can make the predicted label of a certain sample consistent with its neighbor samples.

\begin{figure*}[tbp]
	\centering
	\includegraphics[width = 0.8\textwidth]{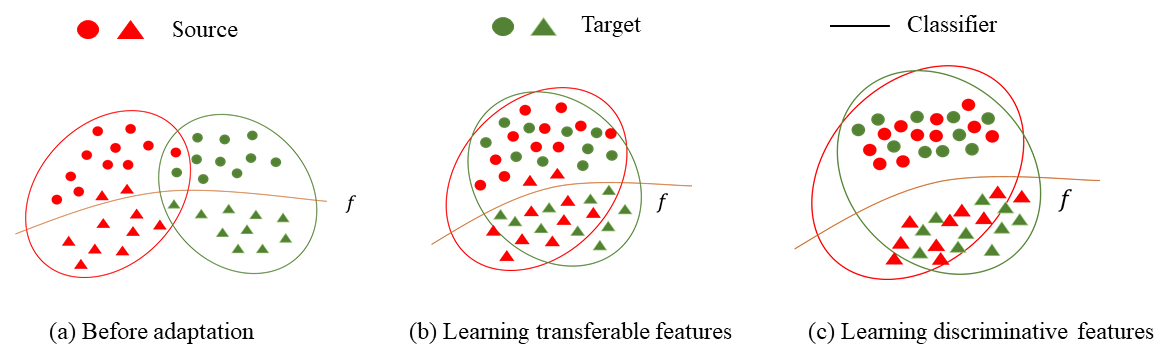}
	\caption{An overview of the proposed method. (a) Before adaptation, the source classifier can not perform well in the target domain. (b)-(c) After learning transferable and discriminative features, the domain discrepancy can be reduced and the samples can be classified correctly.}
	\label{overview11}
\end{figure*}

However, there are two issues with the existing methods. (1) To learn transferable features, MMD distance is a widely used statistic to measure the distribution discrepancy by kernel mean embedding of distributions. However, MMD distance only measures the first-order statistic of different distributions in reproducing kernel Hilbert space (RKHS). Some recent experiments have revealed that the second-order statistic (such as CORAL \cite{ref_12}) is also important to capture useful information for evaluating distribution discrepancy, which is ignored by many methods. (2) To learn discriminative features, previous methods mainly focus on local discriminative information (i.e., sample-level discriminative information), but ignore the global discriminative information (i.e., class-level discriminative information). For example, the classifier trained in the source domain may confuse to distinguish the correct class from a similar class \cite{ref_13}, such as backpack and video-projector. As shown in Fig \ref{exam}(a)-(b), the probability that a source-only model (only trained with labeled source data) misclassifies backpacks as video-projectors in the target domain is over 28\%. This phenomenon is named \textit{class confusion} and it reminds us that the global discriminative information should also be considered.

To overcome these issues, in this paper, we propose a novel method called \textit{learning TransFerable and Discriminative Features for unsupervised domain adaptation} (TFDF), which learns a domain-invariant classifier under the principle of \textit{Structural Risk Minimization} (SRM) to solve the above two issues simultaneously. An overview of the proposed method is shown in Fig 2. For the first issue, we adopt the recently proposed statistic called \textit{Maximum Mean and Covariance Discrepancy} (MMCD) \cite{ref_14} to measure and decrease the distribution discrepancy across domains. MMCD is comprised of MMD and \textit{Maximum Covariance Discrepancy} (MCD). MCD evaluates the Hilbert–Schmidt norm of the difference between covariance operators and can measure the second-order statistic in the RKHS. Therefore, MMCD can consider the first-order and the second-order statistics simultaneously in the RKHS and can capture more distribution information than MMD. For the second issue, we aim to learn more discriminative features at both local and global levels. At the local level, we use the manifold regularization to further exploit the similar geometrical property of the nearest points. At the global level, instead of focusing on the feature space, we concentrate on the label space. We consider the confusion relationship between different classes which is revealed by the inner product of the classifier predictions between different classes (shown in Fig \ref{exam}). The goal is that no examples are ambiguously classified into two classes at the same time. Thus we force the inner product of the same class close to 1 while the different classes close to 0, which encourages the samples in the same class to be more compact while the samples in the different classes to be more dispersed. Thus, TFDF can extract discriminative features.

To sum up, besides minimizing the empirical error in the source domain, TFDF also concentrates on minimizing the distribution discrepancy across domains  to learn transferable features and exploring both global and local discriminative information to learn discriminative features. However, TFDF is a non-convex problem that is difficult to be solved directly, so we firstly propose a variant of TFDF named TFDF-V, which is a convex optimization problem that is easy to be solved with a closed-form solution. Then, taking the solution of TFDF-V as the initial value of TFDF, we use the Adam algorithm \cite{ref_20} (a variant of stochastic gradient descent) to solve the TFDF optimization problem. Comprehensive experiments on five different real-world cross-domain visual recognition datasets are conducted, and the results verify the effectiveness of the proposed algorithm.

\section{Related Work}

\subsection { \textbf{Transferability in domain adaptation}} 

\subsubsection {Shallow domain adaptation}
Shallow domain adaptation methods include instance reweighting and distribution alignment. Instance reweighting based methods assume that the data from the source domain can be reused in the target domain by reweighting samples. Tradaboost \cite{ref_6} is the most representative method which is inspired by Adaboost \cite{ref_16}. The strategy of adjusting the weights of the source and target data is just the opposite, where the source data more conducive to the target data have greater weight in the source domain. LDML \cite{ref_17} also evaluates each sample, and takes full advantage of the pivotal samples, and filters out outliers. DMM \cite{ref_7} learns a transfer support vector machine by extracting invariant feature representations and estimating unbiased instance weights, to jointly minimize the cross-domain distribution discrepancy. However, the performance by instace reweighting is not satisfying.

Distribution alignment based methods focus on finding a feature transformation that projects features of two domains into another common subspace with less distribution discrepancy. The distribution discrepancy across domains includes marginal distribution discrepancy and conditional distribution discrepancy. TCA \cite{ref_8} tries to align marginal distribution across domains, which learns a domain-invariant representation during feature mapping. Based on TCA, JDA \cite{ref_9} tries to align both marginal distribution and conditional distribution simultaneously. Moreover, BDA \cite{ref_18} proposes a balance factor to leverage the importance of different distributions. MEDA \cite{ref_10} can dynamically evaluate the balance factor and has achieved promising performance. The above methods are all based on MMD, which only captures the first-order statistical information across domains. CORAL \cite{ref_12} explores the second-order statistic covariance of the target distribution. Many previous methods only adopt the first-order statitic information while ingore the second-order statistic information. Our method adopts MMCD \cite{ref_14} to evaluate the distribution discrepancy across domains and can capture more useful information for domain adaptation.

\subsubsection {Deep domain adaptation}
 Most deep domain adaptation methods are based on statistical discrepancy minimization. DDC \cite{b37} embeds a domain adaptation layer into the Alexnet \cite{b36} and minimizes \textit{Maximum Mean Discrepancy} (MMD) distance between features of this layer. DAN \cite{b38} minimizes the feature discrepancy between the last three layers of Alexnet \cite{b36} and the multiple-kernel MMD is used to measure the discrepancy. Other measures
are also adopted such as Kullback-Leibler (KL) divergence, Correlation Alignment (CORAL) \cite{b39} which measures the second-order statistical information and Central Moment Discrepancy (CMD) \cite{DBLP:conf/iclr/ZellingerGLNS17} which measures the high-order statistical information. These methods can utilize the deep neural network to extract more transferable features and also have achieved remarkable performance. 

Recently, Inspired by the generative adversarial network \cite{Goodfellow2014GenerativeAN},  adversarial learning is widely used in domain adaptation. DANN \cite{nb18} adopts a domain discriminator to distinguish the source domain from the target domain, while the feature extractor is trained to learn domain-invariant features to confuse the discriminator. ADDA \cite{nb33} designs a symmetrical structure where two feature extractors are adopted. Different from DANN, MCD \cite{nb24} proposes a method to minimize the $H\Delta H$-distance across domains in an adversarial way.

\subsection { \textbf{Discriminability in domain adaptation}}

Learning transferable features may harm the discriminability of the features. Therefore, learning discriminative features is another objective for domain adaptation methods. Inspired by Linear discriminant analysis (LDA) \cite{b31}, some methods take the geometrical information into consideration. For instance, the goal of JGSA \cite{b32} is to minimize the geometrical divergence across domains to enhance the discriminability in shallow domain adaptation. JJDA \cite{b34} extends this idea to deep domain adaptation and considers the instance-level discriminative information. Besides, LPJT \cite{b33} considers the manifold regularization via fisher criterion. ARTL \cite{ref_24} and MEDA also use the manifold regularization via local samples. These methods mainly focus on local discriminative information while the global discriminative information is ignored. TFDF can learn both local and global discriminative information, thus making the features more discriminative.

\section{Method}
\subsection{Problem Definition}
In this paper, we focus on unsupervised domain adaptation. There are a source domain $\mathcal{D}_s = \{(x^1_s, y^1_s),...,(x^{n_s}_s,y^{n_s}_s)\}$ of $n_s$ labeled source samples where $x^i_s \in \mathcal{X}_s, y^i_s \in \mathcal{Y}_s$, and a target domain $\mathcal{D}_t = \{x^1_t, ..., x^{n_t}_t\}$ of $n_t$ unlabled target samples where $x^i_t \in \mathcal{X}_t$. We assume the feature space and label space are the same, i.e., $\mathcal{X}_s = \mathcal{X}_t = \mathbb{R} ^ d$ and $\mathcal{Y}_s = \mathcal{Y}_t  = \{1,2,...,C\}$, while these distributions across domains are different. Especially, we assume the marginal distribution and conditional distribution are different across domains, i.e., $P_s(x_s) \neq P_t(x_t)$ and $Q_s(y_s|x_s) \neq Q_t(y_t|x_t)$. Our goal is to learn a classifier $f: x_t \rightarrow y_t$ to predict $y_t \in \mathcal{Y}_t$ for the target domain $\mathcal{D}_t$ using samples from both domains.

\subsection{Overall Objective}
Transferability and discriminability are two key criteria that characterize the superiority of feature representations to enable domain adaptation \cite{ref_30,ref_24,ref_8,ref_9,ref_10}. Thus, TFDF aims to learn a domain-invariant classifier $f$  based on the principle of \textit{Structural Risk Minimization} (SRM) to learn transferable and discriminative features for the distribution adaptation across domains. As mentioned before, we have four complementary objective functions as follows:
\begin{enumerate}[label=(\arabic*)]
    \item Minimizing the source empirical error of the labeled data in the source domain.
    \item Minimizing the distribution discrepancy across domains to learn transferable features.
    \item Minimizing the manifold regularization to learn local discriminative features.
    \item Minimizing the proposed class confusion loss to learn global discriminative features.
\end{enumerate}

The learning framework of TFDF is then formulated as:
\begin{equation}
\begin{split}
f = \arg \min_{f \in \mathcal{H_K}}R(f,\mathcal{D}_s) + \eta||f||_K^2 + \lambda D_{f}(\mathcal{D}_s,\mathcal{D}_t) 
    + \rho M_{f}(\mathcal{D}_s,\mathcal{D}_t) + \xi C_{f}(\mathcal{D}_s, \mathcal{D}_t)
\label{eq_overall}
\end{split}
\end{equation}
where $K$ is the kernel function induced by $\phi$ such that $\left<\phi(x_i), \phi(x_j)\right> = K(x_i,x_j)$ and $\phi: \mathcal{X} \rightarrow \mathcal{H}$ is the feature mapping function that projects the original feature vector to a Hilbert space $\mathcal{H}$. $R(f,\mathcal{D}_s)$ is the empirical error in the source domain, $||f||_K^2$ is the squared norm of $f$. The term $D_{f}(\mathcal{D}_s,\mathcal{D}_t)$ represents the distribution discrepancy across domains, $M_{f}(\mathcal{D}_s,\mathcal{D}_t)$ is a Laplacian regularization and $C_{f}(\mathcal{D}_s,\mathcal{D}_t)$ is the class confusion loss. $\eta$, $\lambda$, $\rho$ and $\xi$ are the corresponding regularization parameters.

In the next subsections, we introduce each objective separately and give the learning method finally.

\subsection{Source Error Minimization}
The first objective of TFDF is to learn an adaptive classifier that can classify source samples correctly. To begin with, we can induce a standard classifier $f$ on the labeled source samples. According to the structural risk minimization principle \cite{ref_27}, we minimize the source  empirical error as:
\begin{equation}
	R(f,\mathcal{D}_s) + \eta||f||_K^2 = \sum_{i=1}^{n_s}l(f(x_s^i),y_s^i) + \eta||f||_K^2
\label{CLL}
\end{equation}
where $||f||_K^2$ is the squared norm of $f$ in $\mathcal{H_K}$ and $l(,\cdot,)$ is the loss function for classification. In TFDF, the squared loss $l(x_i,y_i) = (y_i - f(x_i))^2$ is used. According to the \textit{Representer Theorem} \cite{ref_28}, the classifier in optimization problem (\ref{eq_overall})  can be represented as 
\begin{equation}
	f(x) = \sum_{i=1}^{n_s + n_t}\beta_i{K}(x_i,x)
\label{cal_F}
\end{equation}
and the equation (\ref{CLL}) can be represented as:
\begin{equation}
\begin{aligned}
	\sum_{i=1}^{n_s}l(f(x_s^i),y_s^i) + \eta||f||_K^2 &= \sum_{i=1}^{n_s + n_t}{A}_{ii}(y_i - f(x_i))^2 + \eta||f||_K^2 
	= ||({Y}-{\beta}^T{K}){A}||_F^2 + \eta \mathrm{tr}({\beta}^T{K}{\beta})
\end{aligned}
\label{RSM}
\end{equation}
where ${A} \in \mathbb{R}^{(n_s + n_t) \times (n_s + n_t)}$ is a diagonal label indicator matrix with ${A}_{ii} =1$ if $x_i \in \mathcal{D}_s$, and ${A}_{ii} = 0$ otherwise. ${Y} \in \mathbb{R}^{C \times (n_s + n_t)}$ is the label matrix with ${Y}_{ij} =1$ if $x_j$ belongs to class $i$, and ${Y}_{ij} = 0$ otherwise. ${K} \in \mathbb{R}^{(n_s+n_t)(n_s+n_t)}$ is kernel matrix, and ${\beta} = (\beta_1, ..., \beta_{n_s + n_t}) \in \mathbb{R}^{(n_s+n_t)\times C}$ are the the parameters of the classifier.
\subsection{Distribution Alignment}
The distribution discrepancy across domains will result in performance degradation when directly applying the classifier trained in the source domain to the target domain. Thus, TFDF aims to learn transferable features to reduce the distribution discrepancy, which includes marginal distribution discrepancy and conditional distribution discrepancy.

\textit{Maximum Mean Discrepancy} (MMD) is a widely used statistic to measure distribution distance, which compares different distributions $p(x)$ and $q(x)$ based on the distances between the sample means of two distributions in a reproducing kernel Hilbert space (RKHS) $\mathcal{H}$, namely
\begin{equation}
	MMD_{\mathcal{H}}^2[\mathcal{H},p,q] = ||E_p\phi(x) - E_q\phi(x)||_{\mathcal{H}}^2
\label{cal_MMD}
\end{equation}

However, MMD only measures the first-order statistic of different distributions. Some experiments have revealed that the second-order statistic (such as CORAL \cite{ref_12}) is also necessary to capture useful information for evaluating distribution discrepancy, which is ignored by existing methods. Recently, a new distribution metric termed \textit{Maximum Mean and Covariance Discrepancy} (MMCD) is proposed in \cite{ref_14}. MMCD considers both the first-order and the second-order statistical information in the RKHS, which is defined as, 
\begin{equation}
MMCD[p,q,\mathcal{H}] = (||E[p] - E[q]||_{\mathcal{H}}^2 + \gamma||C[p] - C[q]||_{\mathcal{HS}}^2)^{\frac{1}{2}}
\end{equation}
where $C[p] = E_{x \sim p}[\phi(x) \otimes \phi(x)] - E_{x \sim p}[\phi(x)]\otimes E_{x \sim p}[\phi(x)]$,  $||\cdot||_{\mathcal{HS}}$ denotes the \textit{Hilbert-Schmidt} norm of the vectors in $\mathcal{HS(H)}$. The empirical estimator of the squared MMCD with classifier $f$ can be given by \cite{ref_14},
\begin{equation}
	\widehat{MMCD}^2[p,q,\mathcal{H}] = \mathrm{tr}({\beta^TKMK^T\beta}) + ||{\beta^TKZK^T\beta}||_F^2 
\end{equation}
where
\begin{equation}
	{M}_{ij} = \begin{cases}
	\frac{1}{n_s^2},  & x_i,x_j \in \mathcal{D}_s\\ 
	\frac{1}{n_t^2}, & x_i,x_j \in \mathcal{D}_t\\ 
	-\frac{1}{n_sn_t}, & \text{otherwise} 
\end{cases}, \quad
{Z}_{ij} = \begin{cases}
	\frac{1}{n_s} - \frac{1}{n_s^2},  & i=j,x_i \in \mathcal{D}_s\\ 
	-\frac{1}{n_s^2}, & i \neq j, x_i,x_j \in \mathcal{D}_s\\ 
	\frac{1}{n_t^2} - \frac{1}{n_t}, & i = j,x_i \in \mathcal{D}_t\\ 
	\frac{1}{n_t^2} & i \neq j, x_i, x_j \in \mathcal{D}_t \\
	0, & \text{otherwise} 
\end{cases}
\end{equation}
Based on MMCD, the distribution discrepancy across domains  can be written as
\begin{equation}
D_{f}(\mathcal{D}_s,\mathcal{D}_t) = (1 - \mu) D_{md}(\mathcal{D}_s,\mathcal{D}_t) + \mu D_{cd}(\mathcal{D}_s,\mathcal{D}_t)
\label{eq_discrepancy}
\end{equation}
where $D_{md}(\mathcal{D}_s,\mathcal{D}_t)$ and $D_{cd}(\mathcal{D}_s,\mathcal{D}_t)$ denote the marginal distribution discrepancy and conditional distribution discrepancy, respectively. $\mu$ is a balance factor between these two discrepancy. The distance of marginal distribution discrepancy is defined as the empirical distances across domains and the distance of conditional probability distributions is defined as the sum of the empirical distances over the class labels between the sub-domains of a same label in the source and target domain,
\begin{equation}
\begin{aligned}
	&D_{md}(\mathcal{D}_s,\mathcal{D}_t) = \widehat{MMCD}^2[\mathcal{D}_s, \mathcal{D}_t,\mathcal{H}] = \mathrm{tr}({\beta^TKM_0K^T\beta}) + ||{\beta^TKZ_0K^T\beta}||_F^2 \\
	&D_{cd}(\mathcal{D}_s,\mathcal{D}_t) = \sum_{c=1}^C\widehat{MMCD}^2[\mathcal{D}_{s,c}, \mathcal{D}_{t,c},\mathcal{H}] = \sum_{c=1}^C(\mathrm{tr}({\beta^TKM_cK^T\beta}) + ||{\beta^TKZ_cK^T\beta}||_F^2)
\end{aligned}
\label{m_c}
\end{equation}
where
\begin{equation}
{({M}_0)}_{ij} = \begin{cases}
\frac{1}{n_s^2}, & x_i,x_j \in \mathcal{D}_s \\
\frac{1}{n_t^2}, &x_i,x_j \in \mathcal{D}_t \\
-\frac{1}{n_sn_t}, & otherwise
\end{cases}, \qquad \quad
{({M}_c)}_{ij} = \begin{cases}
	\frac{1}{n_{s,c}^2}, & x_i,x_j \in \mathcal{D}_{s,c} \\
	\frac{1}{n_{t,c}^2}, &x_i,x_j \in \mathcal{D}_{t,c} \\
	-\frac{1}{n_{s,c}n_{t,c}}, & \begin{cases}
		x_i \in \mathcal{D}_{s,c}, x_j \in \mathcal{D}_{t,c} \\
		x_j \in \mathcal{D}_{s,c}, x_i \in \mathcal{D}_{t,c}
	\end{cases} \\
	0, & otherwise
\end{cases}
\label{cal_M}
\end{equation}
\begin{equation}
{({Z}_0)}_{ij} = \begin{cases}
\frac{1}{n_s} - \frac{1}{n_s^2}, & i=j, x_i \in \mathcal{D}_s \\
-\frac{1}{n_s^2}, & i \neq j,x_i,x_j \in \mathcal{D}_s \\
\frac{1}{n_t^2} -\frac{1}{n_t}, & i = j, x_i \in \mathcal{D}_t \\
\frac{1}{n_t^2} & i \neq j, x_i, x_j \in \mathcal{D}_t \\
0 & otherwise
\end{cases},
{({Z}_c)}_{ij} = \begin{cases}
	\frac{1}{n_{s,c}} - \frac{1}{n_{s,c}^2}, & i=j, x_i \in \mathcal{D}_{s,c} \\
	-\frac{1}{n_{s,c}^2}, & i \neq j,x_i,x_j \in \mathcal{D}_{s,c} \\
	\frac{1}{n_{t,c}^2} -\frac{1}{n_{t,c}}, & i = j, x_i \in \mathcal{D}_{t,c} \\
	\frac{1}{n_{t,c}^2} & i \neq j, x_i, x_j \in \mathcal{D}_{t,c} \\
	0 & otherwise
	\end{cases}
\label{cal_Z}
\end{equation} 
where $\mathcal{D}_{s,c}(\mathcal{D}_{t,c}) = \{x_i|x_i \in \mathcal{D}_s(\mathcal{D}_t), y_i(\hat{y}_i) = c\}$, $y_i$($\hat{y}_i$) is the label (pseudo label) of the sample $x_s^i$($x_t^i$) and $n_{s,c}(n_{t,c}) = |\mathcal{D}_{s,c}(\mathcal{D}_{t,c})|$. As the term $||{\beta^TKZ_cK^T\beta}||_F^2$ is nonconvex, we can approximate the  convex upper bound of the  term in (\ref{m_c}) by using the following theorem:
\begin{theorem}
	Given the constraint that ${\beta^TKHK^T\beta = I}$, the following inequality holds
	\begin{equation}
		||{\beta^T KZ_cK^T\beta}||_F^2 \leq \sigma k 	||{\beta^T KZ_cK^T}||_F^2 
	\end{equation}
	where $k$ is the feature dimensionality and $\sigma = ||({KHK})^{-\frac{1}{2}}||^2$. 
\end{theorem}
where $H = I - \frac{1}{(n_s+n_t)}\mathbf{1}\mathbf{1}^T$ is the centering matrix and the proof is shown in appendix. Note that the constrain that ${\beta^TKHK^T\beta = I}$ is widely used in previous method \cite{ref_8,ref_9,ref_14} to avoid yielding a trivial solution. According to theorem 1, we relax the objective (\ref{eq_discrepancy}) as 
\begin{equation}
	\begin{aligned}
	D_{f}(\mathcal{D}_s, \mathcal{D}_t) &= (1-\mu)\left(\mathrm{tr}({\beta^TKM_0K^T\beta}) + ||{\beta^TKZ_0K^T}||_F^2\right) + \\
	&\mu \left(\sum_{c=1}^C(\mathrm{tr}({\beta^TKM_cK^T\beta}) + ||{\beta^TKZ_cK^T}||_F^2)\right) \\
	&=  (1-\mu)\left(\mathrm{tr}({\beta^TKM_0K^T\beta}) + \mathrm{tr}({\beta^TKZ_0K^TKZ_0K^T\beta})\right) + \\
	&\mu \left(\sum_{c=1}^C(\mathrm{tr}({\beta^TKM_cK^T\beta}) +  \mathrm{tr}({\beta^TKZ_cK^T}{KZ_cK^T\beta}))\right) \\
	\label{MSDI_res}
	\end{aligned}
\end{equation}
\subsection{Local Discriminative Information}
In domain adaptation, we except to learn discriminative features for better classification and adaptation. In this work, local discriminative information refers to the sample-level discriminative information. Manifold regularization is a widely used method to extract local discriminative features.  According to the \textit{manifold assumption} \cite{ref_19}, if two points  are close in the intrinsic geometry, then the corresponding labels are similar. Under this assumption, the \textit{manifold regularization} is computed as 
\begin{equation}
    \begin{aligned}
    &M_{f}(\mathcal{D}_s,\mathcal{D}_t)  = \sum_{i,j=1}^{n_s+n_t}(f(x_i) - f(x_j))^2{W}_{ij} = \sum_{i,j=1}^{n_s+n_t}f(x_i){L}f(x_j) = \mathrm{tr}({\beta^TKLK\beta})
\end{aligned}
\label{eq_mr}
\end{equation}
where ${W}_{i,j}$ is the graph affinity matrix between $x_i$ and $x_j$. ${L} = {I} -{G}^{-\frac{1}{2}}{W}{G}^{-\frac{1}{2}}$ is the graph Laplacian matrix, ${G}$ is a diagonal matrix with ${G}_{ii} = \sum_{j=1}^{n+m}{W}_{ij}$,  ${W}$ is defined as
\begin{equation}
	{W}_{ij} = \begin{cases}
		cos(x_i,x_j), & if \quad x_i \in N_p(x_j) \vee x_j \in N_p(x_i) \\
		0, & otherwise
	\end{cases}
\label{cal_L}
\end{equation}
where $N_p(x_i)$ is the set of $p$-nearest neighbors of $x_i$.

\subsection{Global Discriminative Information}
Maximizing manifold regularization can only learn local discriminative information from the nearest neighbors while the global discriminative information (i.e., class-level discriminative information) is ignored. As shown in Fig \ref{exam}(a)-(b), there exists \textit{class confusion} in domain adaptation methods, which means that the classifier trained in the source domain may confuse to distinguish the correct class from a similar class, such as backpack and video-projector. In order to solve this problem, instead of focusing on the feature space, we concentrate on the label space, where the prediction outputs are able to reveal the class relationships. we use the prediction outputs of the samples in both domains to minimize the class confusion. The prediction of the classifier $f$ on both domains is defined as
\begin{equation}
	{\hat{F}} = f({X})= {\beta^TK}
\end{equation}
where ${\hat{F}} \in \mathbb{R}^{C*(n_s+n_t)}$, we recall that $F_{ij}$ reveals the relationship between the $j$-th example and the $i$-th class. We define
the pairwise class confusion between two classes $i$ and $j$ as
\begin{equation}
	{B_{ij}} = \hat{F}_{i.}^T \cdot \hat{F}_{j.}
\end{equation}
Note that $\hat{F}_{i.}$ denotes the probabilities that the examples in the target domain come from the $i$-th class. The class confusion is defined as the inner product between $\hat{F}_{i.}$ and $\hat{F}_{j.}$. So it measures the possibilities of classifying the examples in the target domain into the $i$-th and the $j$-th classes simultaneously.

Recall that ${B_{ij}}$ well measures the confusion between class $i$ and $j$. As we need to minimize the cross-class confusion, so the ideal situation is that no examples are ambiguously classified into two classes at the same time. In this case, the diagonal elements of $B$ which represent the inner of same classes should be 1 while the off-diagonal elements which represent the inner of different classes should be 0 (as shown in Fig \ref{exam}(a)). Therefore, our goal is to force $B$ to approach the identity matrix. Then, we can define the class confusion objective as
\begin{equation}
\begin{aligned}
	C_{f}(\mathcal{D}_s,\mathcal{D}_t) & = ||{B - I}||_F^2 = ||{\hat{F}\hat{F}^T - I}||_F^2 = ||{\beta^tK}({\beta^t} {K})^T - {I}||_F^2  \\ 
    & =  ||{\beta^TKK^T\beta -I}||_F^2 \\
    & = \mathrm{tr}({\beta^TKK^T\beta\beta^TKK^T\beta - 2\beta^TKK^T\beta + I})
\end{aligned}
\label{MCC}
\end{equation}

Note that previous methods only learn discriminative features at the instance level (local level). However, by enforcing different classes to seperate from each other, the proposed method can learn discriminative features at the class level (global level).

\subsection{Optimization Algorithm}
By substituting equation (\ref{RSM}), (\ref{MSDI_res}), (\ref{eq_mr}), (\ref{MCC}) into equation (\ref{eq_overall}), we can get the optimization problem as follows,
\begin{equation}
\begin{aligned}
	 \beta = \arg\min_{{\beta} } J(\beta) & = 
    ||({Y}-{\beta}^T{K}){A}||_F^2 + \eta \mathrm{tr}({\beta}^T {K}{\beta}) 
    + \lambda \mathrm{tr}(\beta^T {K}V {K}\beta) + \\
	&\rho \mathrm{tr}(\beta^T {K}L{K}\beta) + \xi||\beta^T{KK^T}\beta -{I}||_F^2 \\
	& s.t. \quad \beta^T  {K} H {{K}} \beta = {I}, 
\end{aligned}
\label{optim}
\end{equation}

where \begin{small}$V=(1-\mu)({M_0}+  Z_0  {K}  {K} Z_0) +\mu\ \sum_{c=1}^C({M_c}+  Z_c  {K}  {K} Z_c)$\end{small}.

Equation (\ref{optim}) is an optimization problem with constraints, which is difficult to be solved directly. We relax the problem as an unconstrained optimization problem, namely
\begin{equation}
\begin{aligned}
	 \beta = \arg\min_{{\beta}} J(\beta) &= 
    ||({Y}-{\beta}^T{K}){A}||_F^2 + \eta \mathrm{tr}({\beta}^T {K}{\beta}) + \lambda \mathrm{tr}(\beta^T {K}V {K}\beta) + \\
    & \rho \mathrm{tr}(\beta^T {K}L{K}\beta) +  \xi||\beta^T{KK^T}\beta - {I}||_F^2 + \delta \mathrm{tr}(\beta^T  {K} H {{K}} \beta - {I}), 
\end{aligned}
\label{optim_change}
\end{equation}

\begin{algorithm}[tp]
	\caption{\textbf{TFDF-V} }
	\label{algo_1}
	\begin{algorithmic}[1]
	\REQUIRE Input data $X = [X_s,X_t]$, source labels $Y^s$, iterations number $T$, parameters $\lambda, \rho, \eta, \xi$ and neighbor $p$.
	\ENSURE Initial value $\beta_{init}$.
	\STATE Train a base classifier using $\mathcal{D}_s$ and then apply prediction on $\mathcal{D}_t$ to get its pseudo labels $\hat{Y}_t$.
	\STATE Construct kernel matrix $K$, graph Laplacian matrix $L$ by equation (\ref{cal_L}).
	\FOR{$t = 1,2,...,T$}
	   \STATE Calculate the balance factor $\mu$ using equation (\ref{cal_mu}) and construct MMCD matrix $V$ by equation (\ref{cal_M}),(\ref{cal_Z}) and (\ref{optim}).
	   \STATE Calculate $\beta_{init}$ by solving equation (\ref{cal_init}) and obtain adaptive classifier $f$ by equation (\ref{cal_F}),
	   \STATE Update the pseudo labels of $\mathcal{D}_t$ : $\hat{Y}^t = f(X_t)$.
	\ENDFOR	
	\end{algorithmic}
\end{algorithm}

\begin{algorithm}[tp]
	\caption{\textbf{: TFDF} }
	\label{algo_2}
	\begin{algorithmic}[1]
	\REQUIRE Input data $X = [X_s,X_t]$, source labels $Y^s$, iterations number $T$, learning rate $\alpha$, the parameters $\lambda, \rho, \eta, \xi$, $\delta$ and neighbor $p$. \\
	\textbf{Initialization:} $\theta_1 = 0.9, \theta_2 = 0.999, \epsilon = 10^{-8}$,
	\ENSURE Adaptive classifier $f: X \rightarrow Y$.
	\STATE Train a base classifier using $\mathcal{D}_s$ and then apply prediction on $\mathcal{D}_t$ to get its pseudo labels $\hat{Y}_t$.
	\STATE Construct kernel matrix $K$, graph Laplacian matrix $L$ by equation (\ref{cal_L}).
	\STATE Calculate $\beta_{0} = \beta_{init}$ by solving equation (\ref{cal_init})  for \textbf{TFDF-V}.
	\FOR{$t = 1,2,...,T$}
	   \STATE Calculate the balance factor $\mu$ using equation (\ref{cal_mu}) and construct MMCD matrix $V$ by equation (\ref{cal_M}),(\ref{cal_Z}) and (\ref{optim}).
	   \STATE Get gradients $g_t$ by equation (\ref{get_ge}).
	   \STATE $m_t \leftarrow \theta_1 m_{t-1} + (1- \theta_1)g_t $, $v_t \leftarrow \theta_2 v_{t-1} + (1-\theta_2)g_t^2$, $\widehat{m}_t \leftarrow \frac{m_t}{1-\theta_1^t}$, $\widehat{v}_t \leftarrow \frac{v_t}{1-\theta_2^t}$.
	   \STATE $\beta_{t} = \beta_{t-1} - \alpha \frac{\hat{m}_t}{\sqrt{\hat{v}_t + \epsilon}}$.
	   \STATE Obtain adaptive classifier $f$ by equation (\ref{cal_F}).
	   \STATE Update the pseudo labels of $\mathcal{D}_t$ : $\hat{Y}^t = f(X_t)$.
	\ENDFOR	
	\end{algorithmic}
\end{algorithm} 

Since the fifth term of equation (\ref{optim_change}) is a non-convex fourth-order term, the optimization problem doesn't have the closed-form solution. 
Therefore, we  adopt the  adaptive moment estimation (\textbf{Adam}) algorithm \cite{ref_20} which is a variant of stochastic gradient descent (\textbf{SGD}) to solve $\beta$ iteratively. Take the derivative of $\beta$ and we will get
\begin{equation}
\begin{aligned}
	g_t = \frac{\partial J(\beta)}{\partial\beta} =  & -2{KAY^T} + 2{KAK}\beta + 2\eta {K} \beta + 2\lambda {KVK} \beta\\
	 & + 2\rho {KLK} \beta  +  2\delta {KHK} \beta + 2\xi {KK}\beta\beta^T{KK}\beta - 2\xi {KK}\beta
\end{aligned}
\label{get_ge}
\end{equation}

We experimentally found that it is important to set a proper initial value. Thus, we propose a variant of optimization problem (\ref{optim_change}) which is named as \textbf{TFDF-V}. We let $\xi = 0$ for optimization problem (\ref{optim_change}),  and by setting the derivative of objective function to ${0}$, we can get
\begin{equation}
	\beta_{init} = (({A}+\lambda {V}+\rho {L} + \delta {H}){K}+\eta {I})^{-1}{A}{Y}^T
\label{cal_init}
\end{equation}
Note that \textbf{TFDF-V} can be solved by a closed solution. We firstly run the \textbf{TFDF-V} algorithm to get $\beta_{init}$. Then $\beta_{init}$ is set as the initial value of \textbf{TFDF} algorithm. The detailed pseudo codes of \textbf{TFDF-V} and \textbf{TFDF} are described in algorithm \ref{algo_1} and algorithm \ref{algo_2}, respectively.

\vspace{-3mm}
\section{Experiments and evaluations} 
In this section, we evaluate the performance of TFDF by extensive experiments on five widely-used common datasets. Codes will be available online upon publication.

\subsection{Data Preparation}
We adopt five public image datasets: Office+Caltech, MNIST+USPS, and COIL,  which are popular for domain adaptation methods and have been widely used in previous works. Note that there are no noise and no missing values in these three datasets.

The \textbf{Office-Caltech} dataset \cite{ref_21} consists of images from 10 overlapping object classes between Office31 and Caltech-256.
Specifically, we have four domains, \textbf{C} (\textit{Caltech-256}), \textbf{A} (\textit{Amazon}), \textbf{W} (\textit{Webcam}), and \textbf{D} (\textit{DSLR}).  By randomly selecting two different domains as the source domain and target domain respectively, we construct $3 \times 4 = 12$ cross-domain object tasks, e.g. \textbf{C} $\rightarrow$ \textbf{A}, \textbf{C} $\rightarrow$ \textbf{W},..., \textbf{D} $\rightarrow$ \textbf{W}. Both 800 SURF \cite{ref_9} and 4,096 DeCaf6 \cite{b35} features are used for these datasets.

\textbf{USPS} (U) and \textbf{MNIST} (M) are standard digit recognition datasets containing handwritten digits from 0-9. USPS consists of 7291 training images and 2007
test images of size 16 × 16. MNIST consists of 60000 training images and 10000 test images of size 28 × 28. We construct two tasks: U $\rightarrow$ M and M $\rightarrow$ U. 256 SURF features are used for these datasets.

\textbf{COIL20} contains 20 objects with 1440 images. When the object rotates on the turntable, the object is photographed from different angles every 5 degrees, so each object has 72 images. Each image is 32×32 pixels with 256 gray levels per pixel. Two subsets COIL1 and COIL2 are partitioned from the dataset in \cite{ref_9}. We construct one dataset COIL1 vs COIL2 by selecting all 720 images in COIL1 to form the source data, and all 720 images in COIL2 to form the target data. We construct two tasks: COIL1 $\rightarrow$ COIL2 and COIL2 $\rightarrow$ COIL1.

\subsection{Baselines}
We compare the performance of TFDF with traditional machine learning approaches, several state-of-the-art traditional approaches, and deep domain adaptation approaches: 
\begin{itemize}
	\item traditional machine learning approaches: 1-Nearest Neighbor (\textbf{1NN}), Support Vector Machine (\textbf{SVM}) and Principal Component Analysis (\textbf{PCA}), 
	\item traditional domain adaptation approaches:  Transfer Component Analysis (\textbf{TCA}) \cite{ref_8}, Geodesic Flow Kernel (\textbf{GFK}) \cite{ref_21}, Joint Distribution Alignment (\textbf{JDA}) \cite{ref_9},
     Transfer Joint Matching (\textbf{TJM}) \cite{ref_23}, 
     Adaptation Regularization (\textbf{ARTL}) \cite{ref_24}, 
    CORrelation ALignment (\textbf{CORAL}) \cite{ref_12}, 
 Scatter Component Analysis (\textbf{SCA}) \cite{ref_25}, 
 Joint Geometrical and Statistical Alignment (\textbf{JGSA}) \cite{ref_26},
 Distribution Matching Machine (\textbf{DMM}) \cite{ref_7}, 
  MMCD based Domain Adaptation(\textbf{McDA}) \cite{ref_14}, 
 Manifold Embedded Distribution Alignment (\textbf{MEDA}) \cite{ref_10}, and Confidence-Aware Pseudo Label Selection (\textbf{CAPLS}) \cite{Wang2019UnifyingUD}. 
	\item deep domain adaptation approaches: \textbf{AlexNet} \cite{b36}, Deep Domain Confusion (\textbf{DDC}) \cite{b37}, Deep Adaptation Network (\textbf{DAN}) \cite{b38}, Deep CORAL (\textbf{DCORAL}) \cite{b39}, and a compact DNN (\textbf{DNN}) \cite{Wu2017ACD}.
\end{itemize}

\subsection{Experimental Setup}
For fair comparison and following \cite{ref_8,ref_21}, 1NN, SVM and PCA are trained on the labeled source data, and tested on the unlabeled target data; Other traditional domain adaptation methods (e.g. TCA,JDA) are performed on both the source data and the target data and tested to classify the unlabeled target data.  All the baselines except MEDA are performed in original feature space. While MEDA \cite{ref_10} and TFDF firstly perform \textit{manifold feature learning} to project the original feature to a new feature space with $z = g(x) = \sqrt{G}x$. $G$ can be computed efficiently by singular value decomposition \cite{ref_21}. RBF kernel is used in the experiment. We adopt the method in MEDA \cite{ref_10} to estimate the balance factor $\mu$, namely, 
\begin{equation}
	\mu = \frac{d_M}{d_M + \sum_{c=1}^Cd_c}
\label{cal_mu}
\end{equation}
where $d_A(\mathcal{D}_s,\mathcal{D}_t) = 2(1-2\epsilon(h))$ is $A$-distance \cite{ref_15}, which denotes the error of a linear classifier discriminating the two domains $\mathcal{D}_s$ and $\mathcal{D}_t$. We compute the marginal distribution discrepancy as $d_M = d_A(\mathcal{D}_s,\mathcal{D}_t)$ and the conditional distribution discrepancy as $\sum_{c = 1}^Cd_c = \sum_{c=1}^Cd_A(\mathcal{D}_{s,c}, \mathcal{D}_{t,c})$. Deep methods can be used to the original images and the results of deep domain adaptation methods are directly reported from their original papers wherever available.

Under our experimental setup, it is impossible to tune the optimal parameters using cross validation, since labeled and unlabeled samples are  from different distributions. Thus following previous methods \cite{ref_9}, we evaluate all methods by empirically searching the parameter space for the optimal parameter settings, and report the best results of each method.  We set number of nearest numbers by searching $p \in \{5,10, 15, 20, 30 \} $ and we set adaptation regularization parameter $\lambda, \eta, \rho, \delta$ by searching  $\lambda, \eta, \rho, \delta, \in \{0.001, 0.005, 0.01, 0.05, 0.1, 0.5, 1, 5,10\}$.

In the comparative study of TFDF, we set 1) $\eta = 0.1,\lambda = 10.0, \rho = 0.1, \delta = 0.01, p = 10 $ for COIL dataset, 2) $\eta = 0.1,\lambda = 10.0, \rho = 1.0, \delta = 0.01, p = 10 $ for digital and Office-Caltech dataset. Additionally, we set $T =10$ for TFDF-V and $T =100$ , $\alpha = 0.0005$ for TFDF. The experiments on parameter sensitivity in later experiments (Section \ref{sec_param}) indicate that TFDF stays robust with a wide range of parameter choices. We use classification accuracy on the test data as the evaluation metric, which is widely used in literature \cite{ref_9}: 
\begin{equation}
accuracy = \frac{|x: x \in \mathcal{D}_t \wedge \hat{y}(x) = y(x)|}{|x:x\in \mathcal{D}_t|}. 	
\end{equation}
$y(x)$ and $\hat{y}(x)$ are the ground truth and predicted labels for the target domain samples, respectively.

\subsection{Experimental Results and Analysis}

The results on five  real-world cross-domain (object, digit and object) datasets are shown in Table \ref{res_1}, \ref{res_3} and \ref{res_2}. From these results, we can draw several observations:

\begin{table*}[tbp]
	\centering
	\caption{Accuracy (\%) on Office-Caltech datasets using SURF features.}
	\resizebox{0.97\textwidth}{!}{
		\begin{tabular}{|c|c|c|c|c|c|c|c|c|c|c|c|c|c|c|}
			\hline
			Task & 1NN & SVM & PCA & TCA & GFK & JDA & TJM & CORAL & SCA & ARTL & McDA & MEDA  & {TFDF}\\ \hline \hline
			C $\rightarrow$ A & 23.7 & 53.1 & 39.5 & 45.6 & 46.0 & 43.1 & 46.8 & 52.1 & 45.6 & 44.1 & 43.5 & 56.5 & \textbf{58.0}\\ \hline
			C $\rightarrow$ W & 25.8 & 41.7 & 34.6 & 39.3 & 37.0 & 39.3 & 39.0 & 46.4 & 40.0 & 31.5 & 44.4 & \textbf{53.9} &  52.9\\ \hline
			C $\rightarrow$ D & 25.5 & 47.8 & 44.6 & 45.9 & 40.8 & 49.0 & 44.6 & 45.9 & 47.1 & 39.5 & 50.1 & 50.3 & \textbf{59.2}\\ \hline
			A $\rightarrow$ C & 26.0 & 41.7 & 39.0 & 42.0 & 40.7 & 40.9 & 39.5 & 45.1 & 39.7 & 36.1 & 41.0 & 43.9 & 44.4\\ \hline
			A $\rightarrow$ W & 29.8 & 31.9 & 35.9 & 40.0 & 37.0 & 38.0 & 42.0 & 44.4 & 34.9 & 33.6 & 44.4 & \textbf{53.2}  & 51.2 \\ \hline
			A $\rightarrow$ D & 25.5 & 44.6 & 33.8 & 35.7 & 40.1 & 42.0 & 45.2 & 39.5 & 39.5 & 36.9 & 42.7 & 45.9 & \textbf{46.5}\\ \hline
			W $\rightarrow$ C & 19.9 & 28.8 & 28.2 & 31.5 & 24.8 & 33.0 & 30.2 & 33.7 & 31.1 & 29.7 & \textbf{35.3} & 34.0  & 33.9 \\ \hline
			W $\rightarrow$ A & 23.0 & 27.6 & 29.1 & 30.5 & 27.6 & 29.8 & 30.0 & 36.0 & 30.0 & 38.3 & 37.4 & 42.7 & \textbf{43.1}\\ \hline
			W $\rightarrow$ D & 59.2 & 78.3 & 89.2 & 91.1 & 85.4 & \textbf{92.4} & 89.2 & 86.6 & 87.3 & 87.9 & 89.17 & 88.5   & 90.5\\ \hline
			D $\rightarrow$ C & 26.3 & 26.4 & 29.7 & 33.0 & 29.3 & 31.2 & 31.4 & 33.8 & 30.7 & 30.5 & 34.8 & 34.9 & \textbf{37.2}\\ \hline
			D $\rightarrow$ A & 28.5 & 26.2 & 33.2 & 32.8 & 28.7 & 33.4 & 32.8 & 37.7 & 31.6 & 34.9 & 36.7 & 41.2 & \textbf{42.1}\\ \hline
			D $\rightarrow$ W & 63.4 & 52.5 & 86.1 & 87.5 & 80.3 & 89.2 & 85.4 & 89.8 & 84.4 & 88.5 & 89.8 & 87.5 & \textbf{90.2}\\ \hline \hline
			Average & 31.4 &  41.1 & 43.6 & 46.2 & 43.1 & 46.8 & 46.3 & 48.8 & 45.2 & 44.3 & 49.2 & 52.7 & \textbf{54.1}\\ \hline
		\end{tabular}
	}
\label{res_1}
\end{table*}

\begin{table*}[t]
	\centering
	\caption{Accuracy~(\%) on Office-Caltech datasets using DeCaf6 features.}
	\label{res_3}
	\resizebox{0.97\textwidth}{!}{
		\begin{tabular}{|c|c|c|c|c|c|c|c|c|c|c|c|c|c|c|c|c|c|c|c|c|}
			\hline
			\multirow{2}{*}{Task} & \multicolumn{11}{c|}{Traditional Methods} & \multicolumn{5}{c|}{Deep Methods}  &\multirow{2}{*}{TFDF}\\ \cline{2-17}
			& SVM & PCA & TCA & GFK & JDA  & SCA & ARTL & JGSA & CORAL & DMM  & CAPLS & AlexNet & DDC & DAN & DCORAL &   DNN\\ \hline \hline
			C $\rightarrow$ A & 91.6 & 88.1 & 89.8 & 88.2 & 89.6 & 89.5 & 92.4 & 91.4  & 92.0  & 92.4 & 90.8 & 91.9 & 91.9 & 92.0 & 92.8  & 93.0  & \textbf{93.5}\\ \hline
			C $\rightarrow$ W  & 80.7 & 83.4 & 78.3 & 77.6 & 85.1 & 85.4 & 87.8 & 86.8 & 80.0 & 87.5 & 85.4 & 83.7 & 85.4 & 90.6 & 91.1  & 93.0  &\textbf{95.6} \\ \hline
			C $\rightarrow$ D  & 86.0 & 84.1 & 85.4 & 86.6 & 89.8 & 87.9 & 86.6 & 93.6 & 84.7 & 90.4 & \textbf{95.5} & 87.1 & 88.8 & 89.3 & 91.4  & 91.5  & 93.0 \\ \hline
			A $\rightarrow$ C & 82.2 & 79.3 & 82.6 & 79.2 & 83.6  & 78.8 & 87.4 & 84.9 & 83.2 & 84.8 & 86.1 & 83.0 & 85.0 & 84.1 & 84.7  & 86.5  & \textbf{87.8}  \\ \hline
			A $\rightarrow$ W & 71.9 & 70.9 & 74.2 & 70.9 & 78.3 & 75.9 & 88.5 & 81.0 & 74.6 & 84.7 &  87.1 & 79.5  & 86.1 & 91.8 & -   & \textbf{94.9}  &88.1 \\ \hline
			A $\rightarrow$ D & 80.9 & 82.2 & 81.5 & 82.2 & 80.3 & 85.4 & 85.4 & 88.5 & 84.1 & 92.4 & \textbf{94.9} & 87.4 & 89.0 & 91.7 & -  & 93.3  & 94.2 \\ \hline
			W $\rightarrow$ C & 67.9 & 70.3 & 80.4 & 69.8 & 84.8 & 74.8 & \textbf{88.2} & 85.0 & 75.5 & 81.7 & \textbf{88.2} & 73.0 & 78.0 & 81.2 & 79.3   & 85.9  & 86.5 \\ \hline
			W $\rightarrow$ A & 73.4 & 73.5 & 84.1 & 76.8 & 90.3 & 86.1 & 92.3 & 90.7 & 81.2 & 86.5 & 92.3 & 83.8 & 84.9 & 92.1 & -  & 92.5  & \textbf{93.2} \\ \hline
			W $\rightarrow$ D & 100.0 & 99.4 & 100.0 & 100.0 & 100.0 & 100.0 & 100.0 & 100.0 & 100.0  & 98.7 & 100.0 & 100.0 & 100.0 & 100.0 & -   & 100.0  & 98.7 \\ \hline
			D $\rightarrow$ C & 72.8 & 71.7 & 82.3 & 71.4 & 85.5 & 78.1 & 87.3 & 86.2 & 76.8 & 83.3 & \textbf{88.8} & 79.0 & 81.1 & 80.3 & 82.8   & 83.1  & 87.7 \\ \hline
			D $\rightarrow$ A & 78.7 & 79.2 & 89.1 & 76.3 & 91.7 & 90.0 & 92.7 & 92.0 & 85.5 & 90.7 & 93.0 & 87.1 & 89.5 & 90.0 & -  & 93.3  & \textbf{93.1} \\ \hline
			D $\rightarrow$ W & 98.3 & 98.0 & 99.7 & 99.3 & 99.7 & 98.6 & \textbf{100.0} & 99.7 & 99.3 & 99.3 & \textbf{100.0} & 97.7 & 98.2  & 98.5 & -  & 99.2  &98.0 \\ \hline \hline
			Average & 82.0 & 81.7 & 85.6 & 81.5 & 88.2 & 85.9 & 90.7 & 90.0 & 84.7 & 89.4 & 91.8 & 86.1 & 88.2 & 90.1 & -  & 91.5 &   \textbf{92.5} \\ \hline	
		\end{tabular}
	}
\end{table*}

\begin{table*}[tp]
	\centering
	\caption{Accuracy (\%) on USPS+MNIST and COIL datasets.}
	\label{tb-mnist}
	\resizebox{0.97\textwidth}{!}{
		\begin{tabular}{|c|c|c|c|c|c|c|c|c|c|c|c|c|c|c|c|c}
			\hline
			Task & 1NN & SVM & PCA & TCA & GFK & JDA & TJM & CORAL & SCA & ARTL & JGSA & McDA & MEDA &{TFDF} \\ \hline \hline
			U $\rightarrow$ M & 44.7 & 62.2 & 45.0 & 51.2 & 46.5 & 59.7 & 52.3 & 30.5 & 48.0 & 67.7 & 68.2 & - & 72.1 &  \textbf{80.6}\\ \hline
			M $\rightarrow$ U & 65.9 & 68.2 & 66.2 & 56.3 & 61.2 & 67.3 & 63.3 & 49.2 & 65.1 & 88.8 & 80.4 & - & 89.5 &  \textbf{89.9} \\ \hline
			COIL1 $\rightarrow$ COIL2 & 83.6 & 84.7 & 84.7 & 88.5 & 72.5 & 89.3 & 87.6 & 82.64  &-& 88.3 & - & \textbf{94.1} & 90.1 &  93.6\\ \hline
			COIL2 $\rightarrow$ COIL1 & 82.8 & 82.9 & 84.0 & 85.8 & 74.2 & 88.5 & 87.4 & 82.36  & - & 84.0 & - &89.6 & 87.1 &  \textbf{91.8}\\ \hline \hline
			Average & 69.3& 74.5 & 70.0 & 70.5 & 63.6 & 76.2 & 72.7 & 61.2 & - & 82.2 & - & -& 84.7 & \textbf{89.0}\\ \hline
		\end{tabular}
	}
\label{res_2}
\end{table*}

Firstly, TFDF achieves the best performance in most tasks (7/12 tasks) on the Office-Caltech dataset (SURF features). The average accuracy of TFDF on the Office-Caltech dataset (SURF features) is 54.1\%, while the best baseline is MEDA with 52.7\%. Compared with MEDA, the average performance is improved by 1.4\%. The observations on the COIL and digital datasets are the same and the average performance improvement is 2.7\% on five datasets. Since these results are obtained from a large number of datasets, it can convincingly verify that TFDF can build a robust adaptive classifier while reducing cross-domain discrepancy.

Secondly, both TFDF and McDA adopt MMCD to measure domain discrepancy, and they perform better than TCA, JDA, CORAL, and ARTL which either only consider the first-order statistical information or only consider the second-order statistical information. This improvement indicates that considering both the first-order and the second-order statistics simultaneously can capture more information for reducing cross-domain discrepancy. TFDF outperforms McDA because the proposed methods not only adopt MMCD to measure and decrease domain discrepancy but also consider the discriminative information of features. Compared with MEDA, which also performs manifold regularization to learn local discriminative features, minimizing class confusion loss is helpful to learn global discriminative features, thus achieving better performance. Compared with CAPLS, which focus on pesudo label selection within the distribution alignment process, TFDF also achieves better performance although TFDF use all the pseudo labels (even with wrong pseudo labels), which shows the robustness of the proposed method. Moreover, the error matrixes of different algorithms are shown in Fig \ref{exam}, which shows that TFDF can learn global discriminative information to avoid class confusion and get better performance.

Thirdly, TFDF also performs better than the deep methods (AlexNet, DDC, DAN,  DCORAL, and DNN) on Office+Caltech10 datasets. Based on the powerful features extracted from deep models, some traditional methods such as DMM, MDSI-V, and TFDF can achieve better performance than deep methods. As we can see, TFDF achieves the best performance with a  1\% improvement compared to DNN.

\subsection{Effectiveness Analysis}
\subsubsection{Ablation Study}  We conduct an ablation study to analyze how different components of our work contribute to the final performance. When learning the final classifier, TFDF involves four components: Structural Risk Minimization (SRM), Distribution Alignment (DA), Local Discriminative information (LD), and Global Discriminative information (GD). We empirically evaluate the importance of each component. To this end, we investigate different combinations of four components and report average classification
accuracy on five datasets in Table \ref{ab_study}. The first setting where only SRM is used is actually the source-only method where no adaptation is performed and this setting performs the worst. It can be observed that the accuracy is further improved after performing Distribution Alignment (DA). And the use of both global and local discriminative information improves the performance significantly on all datasets. Finally, a combination of all components can achieve the best results.

\begin{table}[h]
	\centering
	\caption{Results of ablation study(accuracy~(\%))}
	\label{tb-decaf}
	\resizebox{0.75\textwidth}{!}
	{
		\begin{tabular}{|c|c|c|c|c|c|c|}
			\hline
			\multicolumn{4}{|c|}{Method} & \multirow{2}{*}{Office-Caltech (surf features)} & \multirow{2}{*}{COIL}& \multirow{2}{*}{MNIST-USPS} \\ \cline{1-4}
			 SRM & DA & LD & GD &  &  &  \\ \hline 
			 \checkmark & $\times$ & $\times$ & $\times$ & 49.93 & 82.22 &  54.83\\ \hline
			 \checkmark & \checkmark & $\times$ & $\times$ & 51.92  & 90.29 &  74.89\\ 
			 \checkmark & $\times$ &  \checkmark &  $\times$ & 51.60 & 83.26 &  56.50\\ 
			 \checkmark & \checkmark & \checkmark & $\times$ & 52.20 & 91.12 & 80.86 \\ \hline
			 \checkmark & \checkmark & \checkmark & \checkmark & \textbf{54.10} &  \textbf{92.71}  &  \textbf{85.22}\\ \hline
		\end{tabular}}
	\label{ab_study}
\end{table}

\subsubsection{Distribution Distance} 
We run JDA, MEDA, and TFDF on task $C \rightarrow D$ (SURF features) using their optimal parameter settings. Then we compute the aggregate MMD distance and MMCD distance of each method on their induced embeddings by equation (\ref{cal_MMD}) and equation (\ref{eq_discrepancy}), respectively. To compute the true distance in both the marginal and conditional distributions across domains, we have to use the ground truth labels instead of the pseudo labels. However, the ground truth target labels are only used for verification, not for learning procedures.

 \begin{figure}[tbp]
	 \centering
	 \subfigure[MMD distance w.r.t. \#iterations]{
	 \includegraphics[width = 0.29\textwidth]{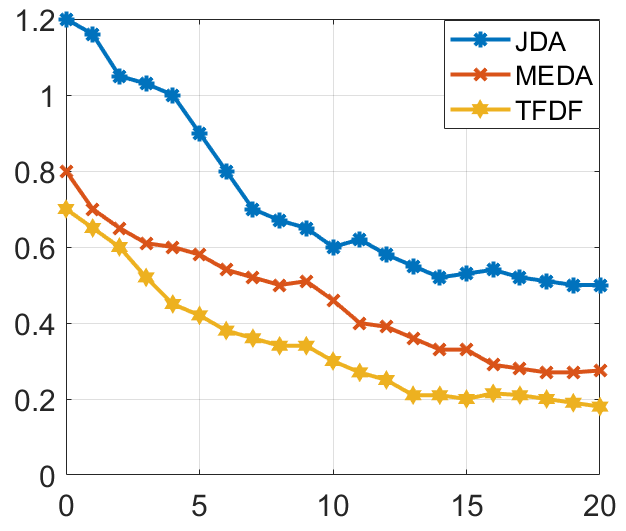}
	 }
	 \subfigure[MMCD distance w.r.t. \#iterations]{
	 \includegraphics[width = 0.287\textwidth]{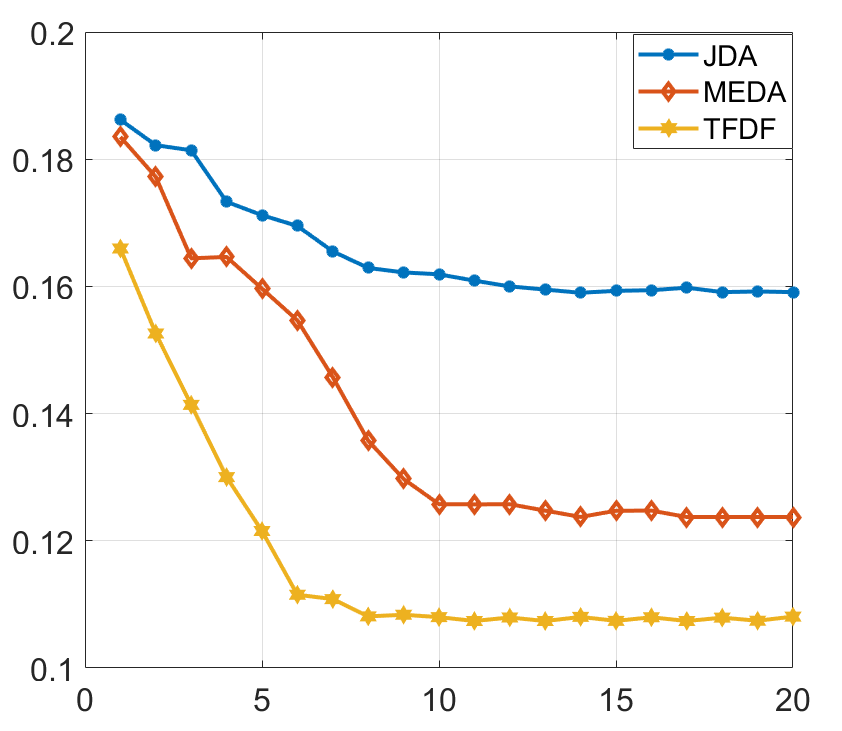}
	 }
	 \subfigure[Accuracy(\%) w.r.t. \#iterations]{
	 \includegraphics[width = 0.305\textwidth]{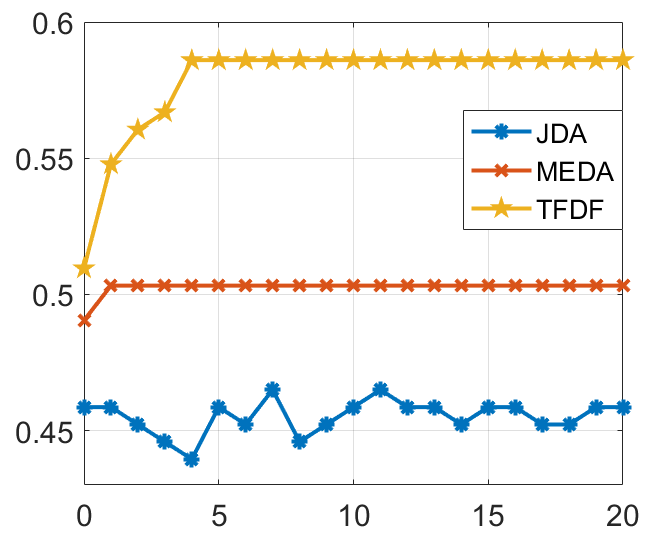}
	 }
 \caption{MMD distance, classification accuracy on the task $C \rightarrow D$}
 \label{acc}
 \end{figure}

Fig \ref{acc}(a) and  Fig \ref{acc}(b) show the MMD distance and MMCD distance computed for each method and Fig \ref{acc}(c) shows the prediction accuracy for each method. As we can see, MEDA and JDA can reduce the domain discrepancy which can be measured by either MMD distance or MMCD distance and achieve good performance in the target domain while TFDF achieves better results. JDA only considers the first-order statistic, while TFDF adapts both the first-order and the second-order statistics. Besides, by minimizing class confusion, TFDF can learn more discriminative features and improve performance.

\subsubsection{Feature Visualization.} In Fig \ref{t_sne}, we visualize the feature representations of task USPS $\rightarrow$  MNIST (\textbf{U}$\rightarrow$\textbf{M}) (10 classes) by t-SNE\cite{b35} using JDA, MEDA and TFDF. Before adaptation, we can see that there is a large distribution discrepancy across domains. After adaptation, JDA learns invariant features which can reduce distribution discrepancy. MEDA further considers the dynamic factor between a marginal distribution and conditional distribution and makes a better adaptation. TFDF not only learns transferable features but also learns local and global discriminative features. Therefore, besides a small distribution discrepancy, the features in both domains are more discriminative and cab be easily classified by the classifier.

\subsubsection{Time Complexity} 
We run JDA, MEDA, and TFDF on datasets Office-Caltech, COIL, and USPS+MNIST using their optimal parameter settings. Then we compute the running time of each method and the results are shown in Table 5. As we can see, JDA and MEDA need less time than TFDF as they are solved by a closed-form solution. Although TFDF-V can be solved by a closed-form solution, TFDF needs more iterations to get the final results than these two baselines, which is a shortcoming of the proposed method.

\begin{table}[h]
	\centering
	\resizebox{0.6\textwidth}{!}
	{
		\begin{tabular}{cccc}
		\toprule
		& Office-Caltech (surf features) & COIL  & MNIST-USPS \\
		\midrule
		JDA  & 5.66 min          & 0.94 min & 8.08 min       \\
		MEDA & 3.51 min          & 0.73 min & 3.27 min      \\
		TFDF & 33.10 min         & 7.86 min & 52.42 min     \\
		\bottomrule
		\end{tabular}
	}
\caption{Runing time on different tasks}
\label{running_time}
\end{table}

\subsection{Parameter Sensitivity}

\label{sec_param}
 In this section, we evaluate TFDF with a wide range of values for regularization parameters $\rho$, $\lambda$, $\eta$ and neighbors number $p$. We only report the results on MNIST $\rightarrow$  USPS (M $\rightarrow$ U), COIL 1 $\rightarrow$ COIL2 and C $\rightarrow$ D tasks, while similar trends on the other tasks are not shown due to space limitation. The results are shown in Fig \ref{param}. It can be observed that TFDF achieves a robust performance with regard to a wide range of parameter values.  Specifically, $p \in [8,16]$, $\rho \in [0.1,1]$, $\lambda \in [1,10]$ and $\eta \in [0.05,0.1]$ are the optimal parameter values.

\section{Conclusion}
In this paper, we propose a method called \textit{learning TransFerable and Discriminative Features for unsupervised domain adaptation} (TFDF), which could learn both transferable and discriminative features simultaneously. On the one hand,  we adopt a recently proposed statistic called MMCD to measure domain discrepancy, which can capture both the first-order and the second-order statistical information simultaneously, thus more statistical information can be explored than MMD-based methods. On the other hand, we propose to learn both local and global discriminative features through manifold regularization and proposed class confusion loss respectively. With the principle of empirical risk minimization, TFDF also integrates the source classification error with the above objectives into a uniform optimization problem. Comprehensive experiments are conducted and the results verified the effectiveness of the proposed method.

\begin{figure*}[t]
	\centering
	\subfigure[t-sne before adaptation]{
	\includegraphics[width = 0.232\textwidth]{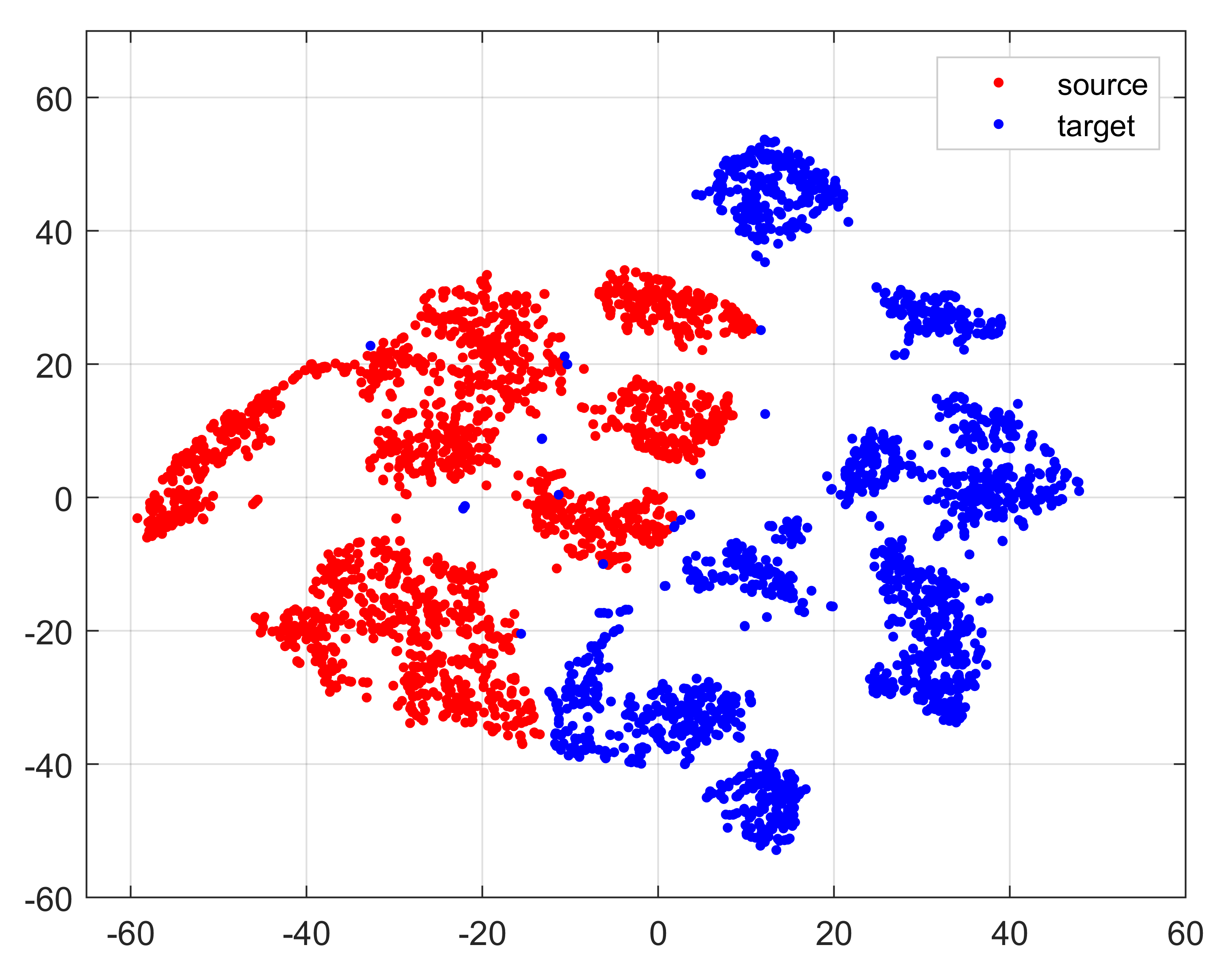}
	}
	\subfigure[t-sne by JDA]{
	\includegraphics[width = 0.232\textwidth]{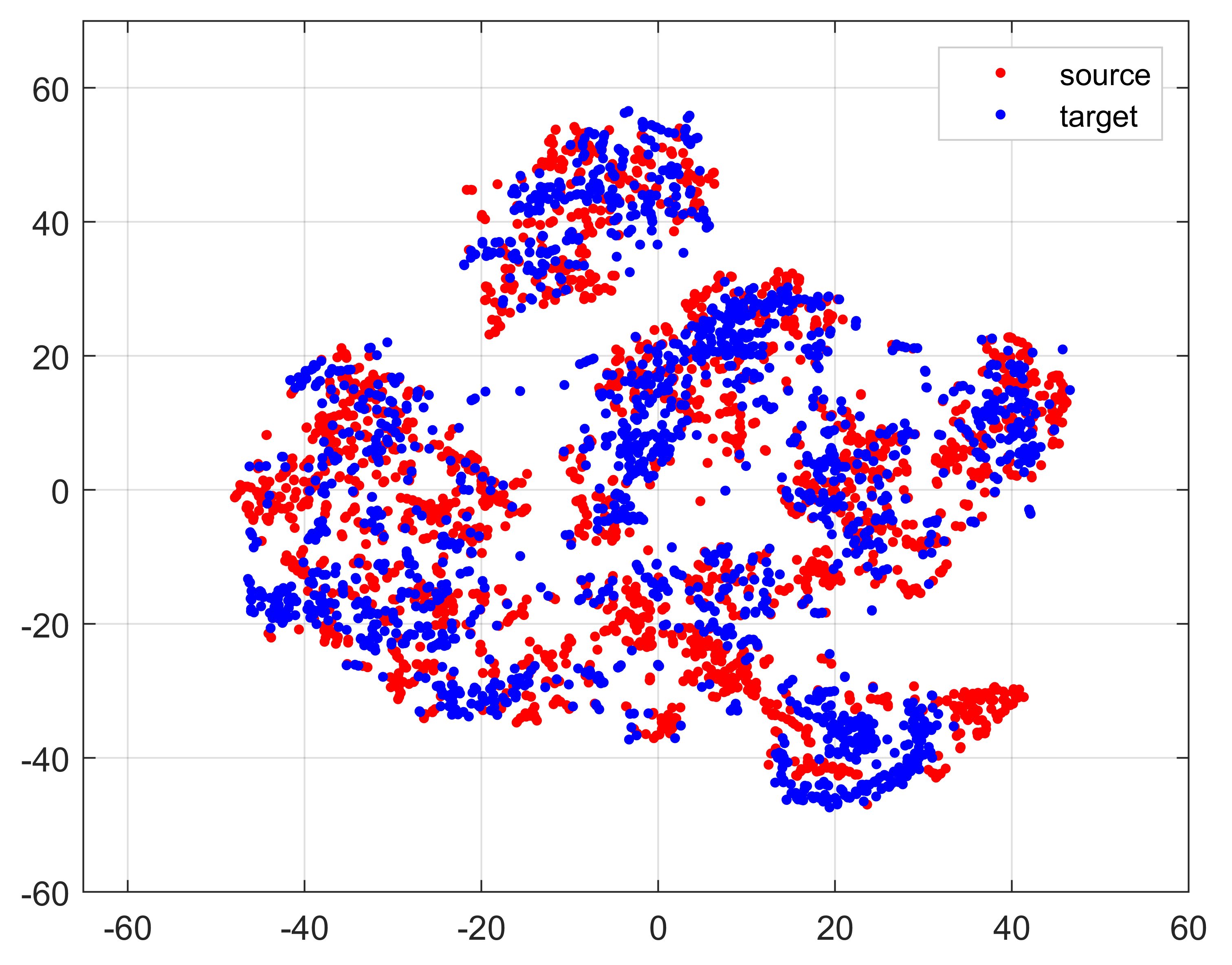}
	}
	\subfigure[t-sne by MEDA]{
	\includegraphics[width = 0.23\textwidth]{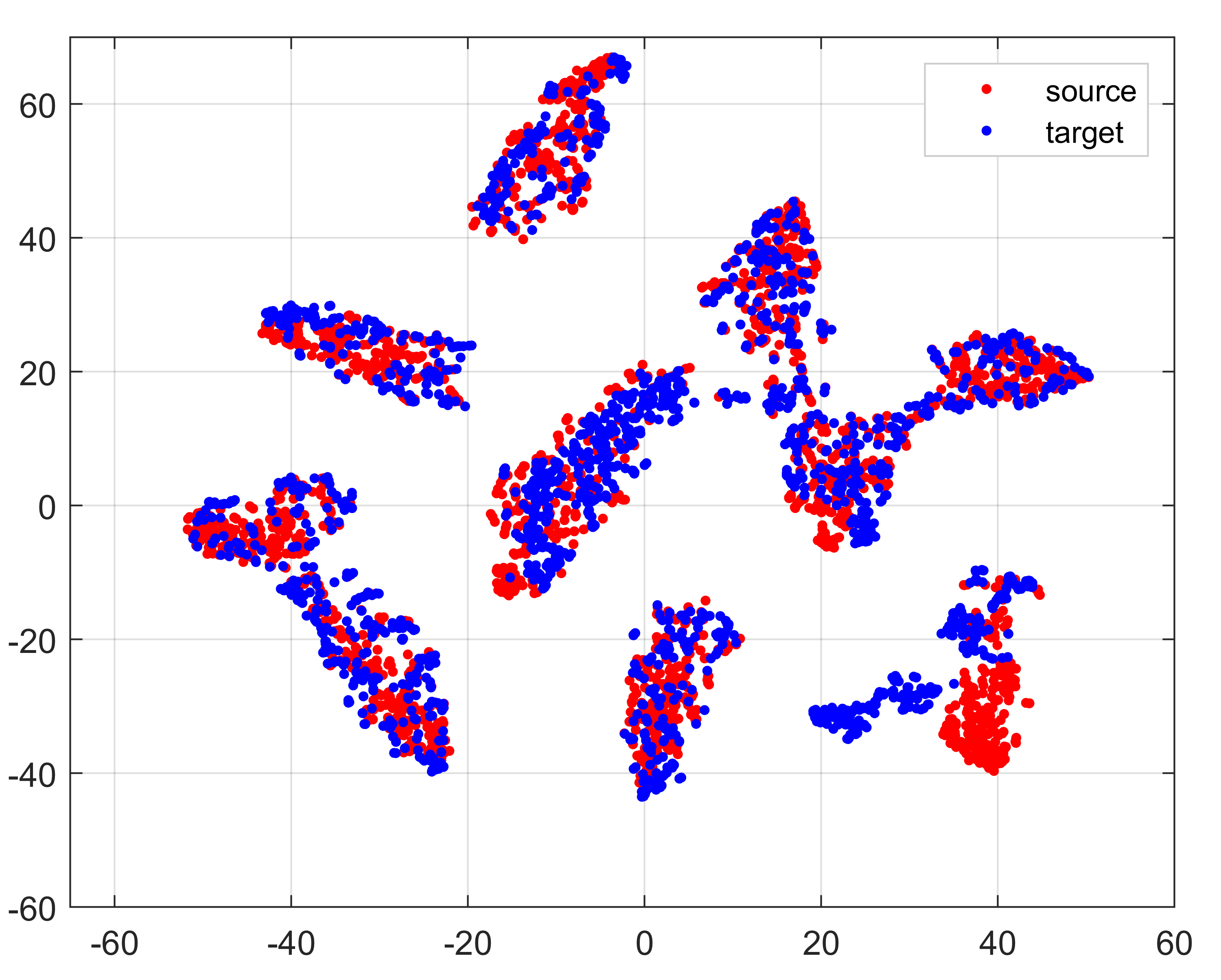}
	}
	\subfigure[t-sne by TFDF]{
	\includegraphics[width = 0.23\textwidth]{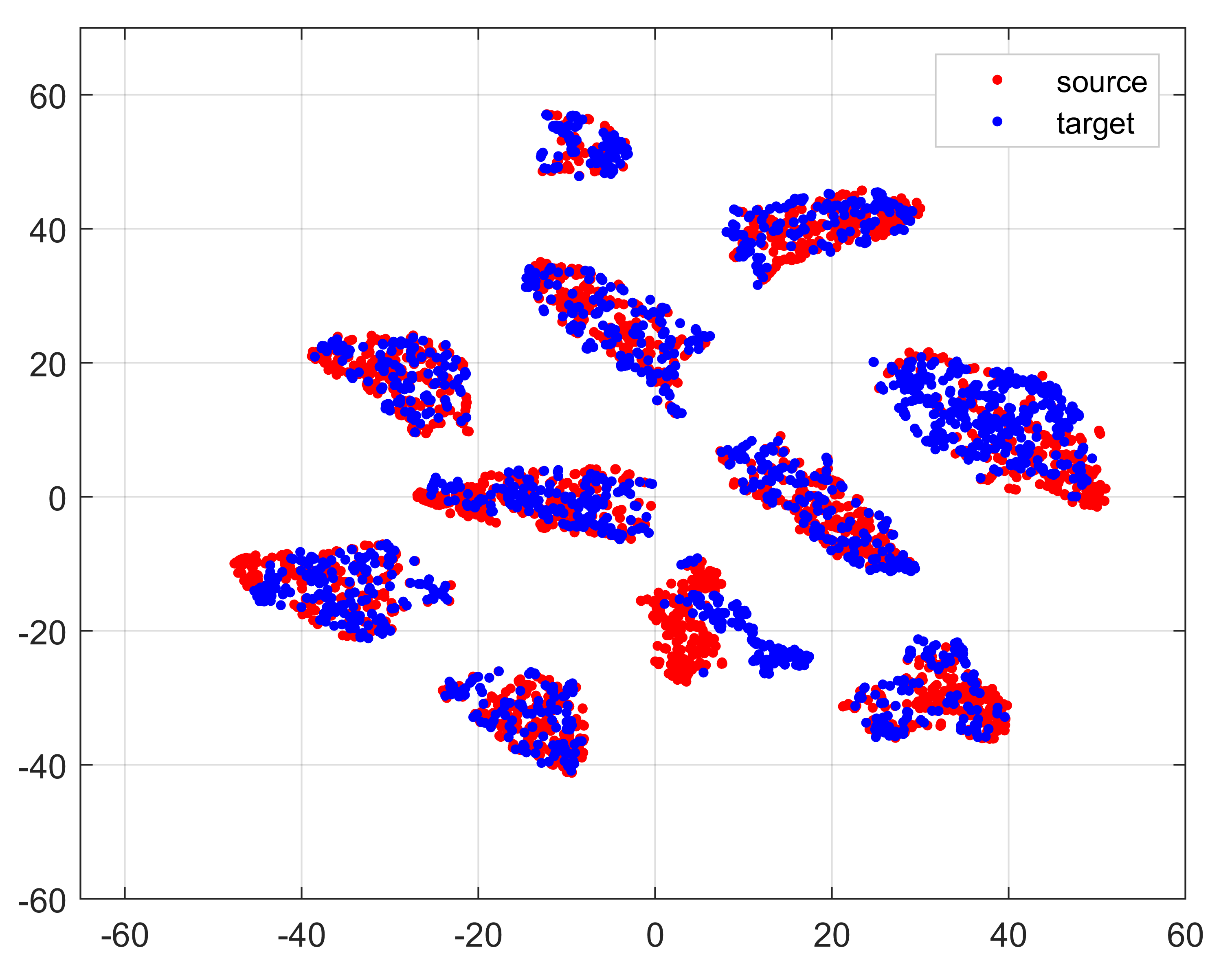}
	}
\caption{Feature visualization on task $U \rightarrow M$}
\label{t_sne}
\end{figure*}

\begin{figure*}[tbp]
	\centering
	\subfigure[neighbors $p$]{
	\includegraphics[width = 0.215\textwidth]{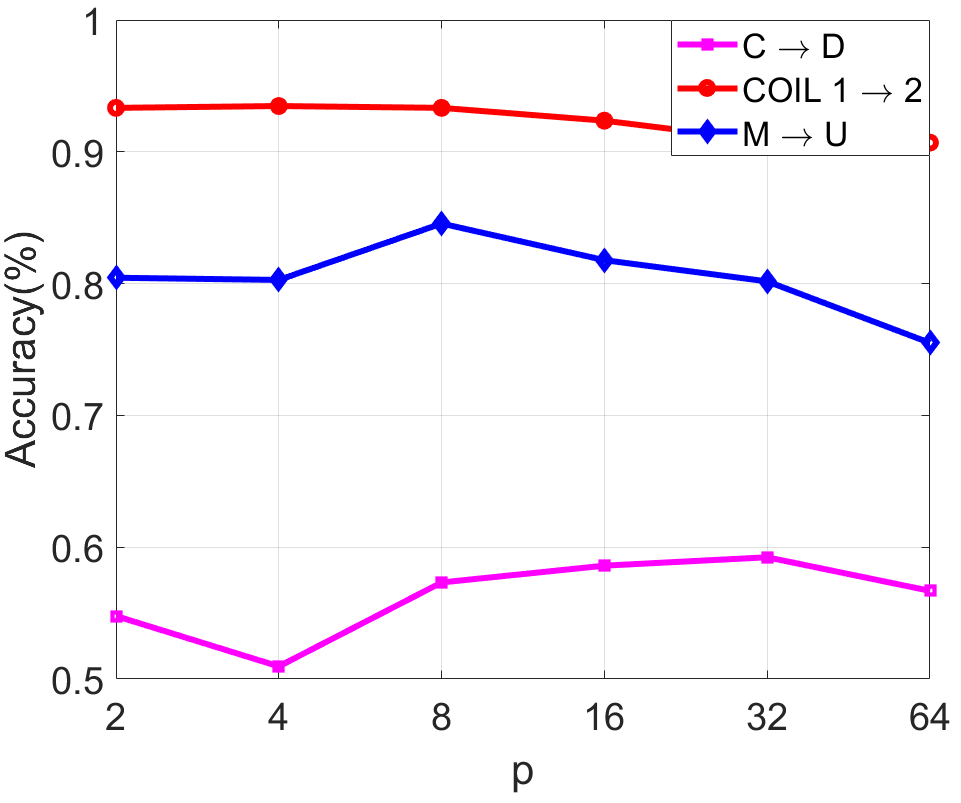}
	}
	\subfigure[$\rho$]{
	\includegraphics[width = 0.24\textwidth]{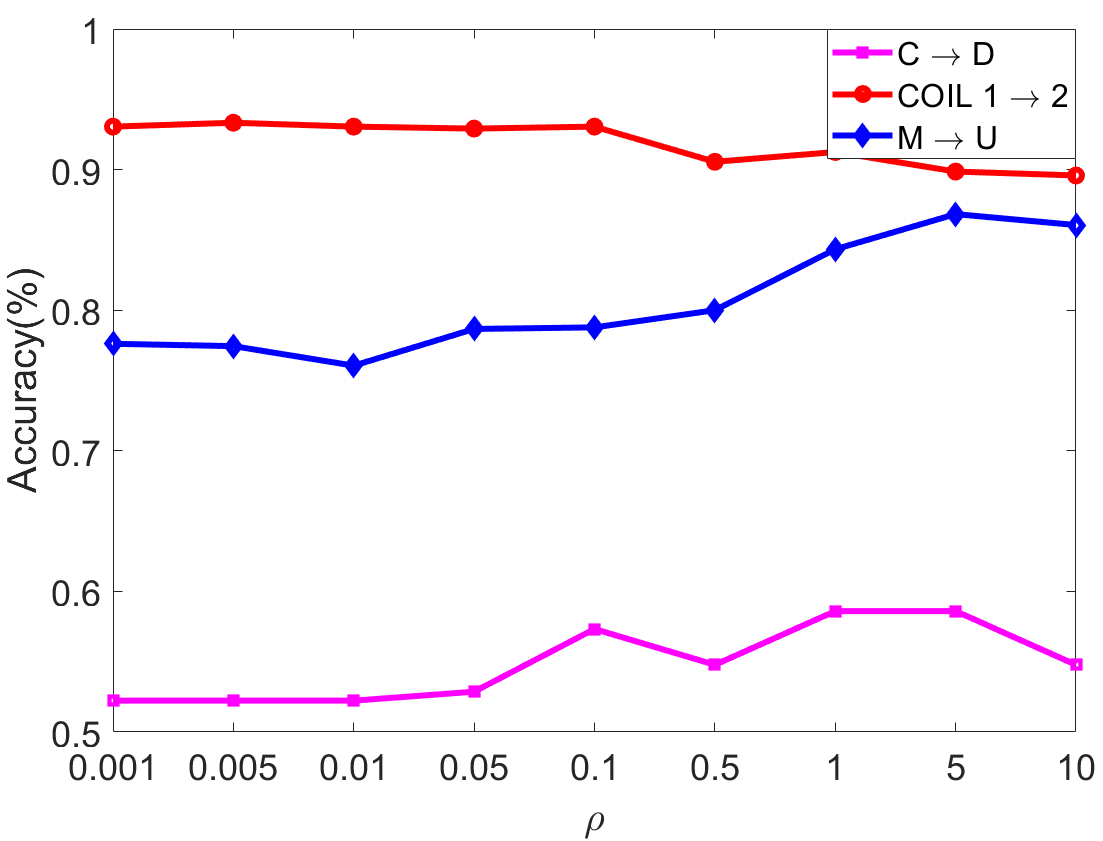}
	}
	\subfigure[$\lambda$]{
	\includegraphics[width = 0.23\textwidth]{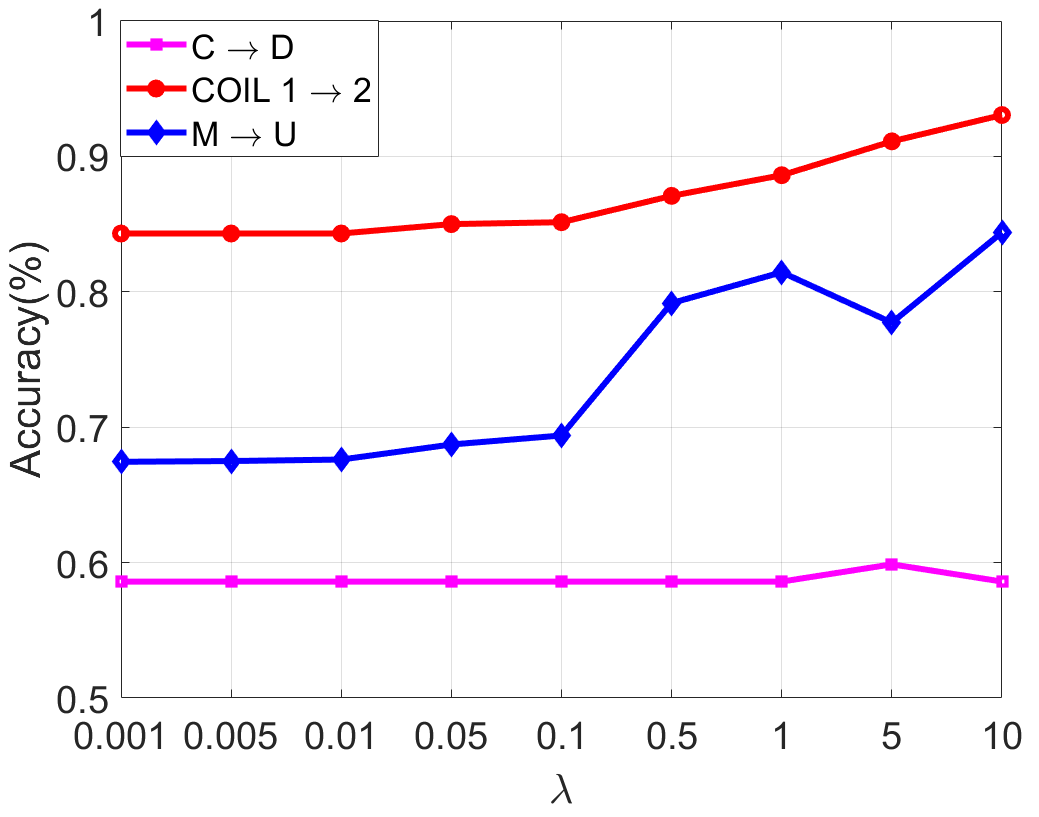}
	}
	\subfigure[$\eta$]{
	\includegraphics[width = 0.23\textwidth]{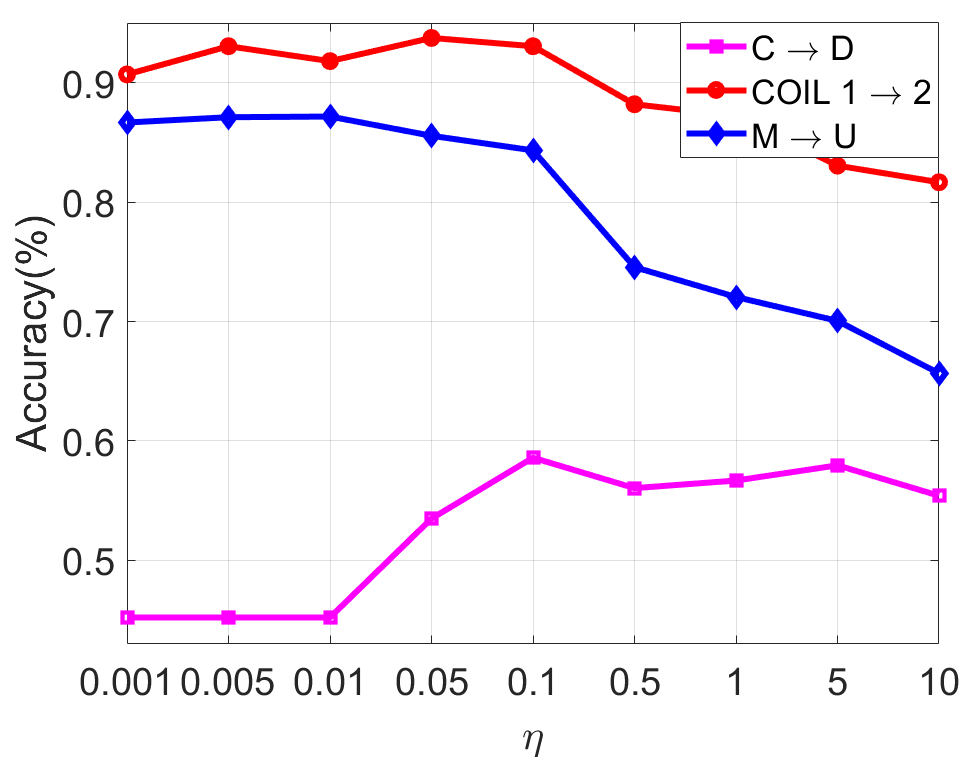}
	}
\caption{Parameter sensitivity on three tasks.}
\label{param}

\end{figure*}

\section{Acknowledgement}

This paper is supported by the National Key Research and Development Program of China (Grant No. 2018YFB1403400), the National Natural Science Foundation of China (Grant No. 61876080, No. 62002137), the Key Research and Development Program of Jiangsu(Grant No. BE2019105), the Collaborative Innovation Center of Novel Software Technology and Industrialization at Nanjing University.

\nocite{*} 
\bibliographystyle{ios1}           
\bibliography{ida}        

%

\appendix
\section{Proof of theorem 1}
According to \cite{ref_14}, we can approximate the  convex upper bound of the second $ ||{\beta^TKZ_0K^T\beta}||_F^2$ and the fourth $||{\beta^TKZ_cK^T\beta}||_F^2$ terms by using the following theorem:
\begin{theorem}
	given the constraint that ${\beta^TKHK^T\beta = I}$, the following inequality holds
	\begin{equation}
		||{\beta^T KZ_cK^T\beta}||_F^2 \leq \sigma k 	||{\beta^T KZ_cK^T}||_F^2 
	\end{equation}
	where $k$ is the feature dimensionality and $\sigma = ||({KHK})^{-\frac{1}{2}}||^2$. 
\end{theorem}
\begin{proof}
    \begin{equation*}
        \begin{split}
    ||{\beta^T KZ_cK^T\beta}||_F^2 &= ||{\beta^TKZ_cK^T(KHK^T)^{-1/2}(KHK^T)^{1/2}\beta}||_F^2  \\
    & \leq ||{\beta^TKZ_cK^T(KHK^T)^{-1/2}}||_F^2 ||{(KHK^T)^{1/2}\beta}||_F^2 \\
    & = k||{\beta^TKZ_cK^T(KHK^T)^{-1/2}}||_F^2 \\
    & \leq k||{\beta^TKZ_cK^T}||_F^2 ||{(KHK^T)^{-1/2}} ||_F^2 \\
    & = k\sigma  ||
    {\beta^TKZ_cK^T|}|_F^2 
        \end{split} 
	\end{equation*}
\end{proof}

The first equation holds because $KHK^T$ is semi-definite positive, while the second and the fourth inequalities follow the Cauchy-Schwarz inequality. In terms of the constraint of theorem 1, $k = \mathrm{tr}({\beta^TKHK^T\beta}) = \mathrm{tr}({I_k})$.

\end{document}